\documentclass[letterpaper, 10 pt, conference]{ieeeconf}  

\IEEEoverridecommandlockouts                              

\usepackage{adjustbox}
\usepackage{graphicx}
\usepackage{graphics} 
\usepackage{dsfont}
\usepackage{epsfig} 
\usepackage{amsbsy}
\usepackage{dsfont}
\usepackage{mathtools}
\usepackage{bm}
\usepackage{subfiles}
\usepackage[noend,ruled,linesnumbered]{algorithm2e}
\usepackage{caption}
\usepackage{subcaption}
\usepackage{cite}
\usepackage{amsmath, amssymb}
\usepackage{bbm}
\usepackage{cite}
\usepackage{hyperref}
\DeclareCaptionFont{mysize}{\fontsize{8}{9.6}\selectfont}
\usepackage{adjustbox}
\usepackage{tabularx}
\newtheorem{assumption}{\textbf{Assumption}}

\newtheorem{lemma}{\textbf{Lemma}}

\newtheorem{theorem}{\textbf{Theorem}}
\newtheorem{definition}{\textbf{Definition}}

\newtheorem{cor}{\textbf{Corollary}}

\usepackage{booktabs}
\usepackage{graphicx}
\usepackage{graphics} 
\usepackage{dsfont}
\usepackage{epsfig} 
\usepackage{amsbsy}
\usepackage{dsfont}
\usepackage{mathtools}
\usepackage{bm}
\usepackage{subfiles}
\usepackage[noend,ruled,linesnumbered]{algorithm2e}
\usepackage{caption}
\usepackage{subcaption}
\usepackage{cite}
\usepackage{amsmath, amssymb}
\usepackage{bbm}
\usepackage{cite}
\usepackage{hyperref}
\DeclareCaptionFont{mysize}{\fontsize{8}{9.6}\selectfont}
\usepackage{adjustbox}
\usepackage{tabularx}
\usepackage{tikz}
\usepackage{pgfplots}
\pgfplotsset{compat=1.18}
\usetikzlibrary{arrows.meta,calc,positioning,decorations.pathreplacing}

\usepackage{graphics} 
\usepackage{epsfig} 
\usepackage{mathptmx} 
\usepackage{times} 
\usepackage{amsmath} 
\usepackage{amssymb}  
\usepackage{dsfont}

\usepackage{tikz}
\usepackage{pgfplots}
\pgfplotsset{compat=1.18}
\usetikzlibrary{arrows.meta,calc,positioning,decorations.pathreplacing}

\title{eCP: Equivariant Conformal Prediction with pre-trained models}
\author{Nikolaos Bousias$^{1}$, 
Lars Lindemann$^{2,3}$ and George Pappas$^{1}$ 
\thanks{$^{1}$ GRASP Laboratory, Department of Electrical \& Systems Engineering, University of Pennsylvania {\tt\small nbousias@seas.upenn.edu}}
\thanks{$^{2}$ Department of Computer Science, University of Southern California}
\thanks{$^{3}$ Automatic Control Laboratory, ETH Zürich}
}
\allowdisplaybreaks
\begin{document}

\maketitle
\thispagestyle{empty}
\pagestyle{empty}

\begin{abstract}

Conformal prediction, a post-hoc, distribution-free, finite-sample method of uncertainty quantification that offers formal coverage guarantees under the assumption of data exchangeability. Unfortunately, the resulting uncertainty regions can grow significantly in long horizon missions, rendering the statistical guarantees uninformative. To that end, we propose infusing CP with geometric information via group-averaging of the pretrained predictor to distribute the non-conformity mass across the orbits. Each sample now is treated as a representative of an orbit, thus uncertainty can be mitigated by other samples entangled to it via the orbit inducing elements of the symmetry group. Our approach provably yields contracted non-conformity scores in increasing convex order, implying improved exponential-tail bounds and sharper conformal prediction sets in expectation, especially at high confidence levels. We then propose an experimental design to test these theoretical claims in pedestrian trajectory prediction.

\end{abstract}

\section{Introduction}

Modern machine learning systems are increasingly deployed in settings where reliable uncertainty quantification is essential, such as robotics, autonomous navigation, and long-horizon forecasting. Conformal prediction (CP) has emerged as a principled and widely adopted framework for this purpose, providing finite-sample, distribution-free prediction sets with formal coverage guarantees under the mild assumption of data exchangeability \cite{vovk2005algorithmic,shafer2008tutorial}. Crucially, CP is post-hoc and model-agnostic: it can be applied on top of any pretrained predictor without retraining or architectural modification. Split conformal prediction extends this framework to modern supervised learning settings by separating training and calibration data \cite{lei2018distribution}.

Despite these strengths \cite{angelopoulos2021gentle}, a well-known limitation of conformal prediction is efficiency. While coverage is guaranteed by design, the resulting prediction sets can become excessively large---particularly at high confidence levels or over long horizons---rendering the uncertainty estimates practically uninformative. This issue is especially pronounced in sequential decision-making and trajectory prediction tasks, where compounding uncertainty quickly leads to overly conservative prediction regions. Recent work has proposed adaptive, weighted, and conditional variants of CP to improve efficiency without sacrificing coverage~\cite{romano2019conformalized,tibshirani2019covariate,gibbs2021adaptive}. More broadly, conformal methods have rapidly expanded across regression, classification, and structured prediction, with modern surveys documenting both their theoretical foundations and practical applications \cite{angelopoulos2021gentle}.

Uncertainty-aware prediction has become particularly important in robotics and safe planning. In pedestrian-rich and dynamically evolving environments, planners must account for multi-modal and long-horizon uncertainty while preserving formal safety guarantees. Several works have explored conformal prediction for sequential and multi-step forecasting problems relevant to this setting. Sun and Yu~\cite{sun2023copula} propose copula-based conformal prediction to model dependencies across time in multi-step time series prediction, producing calibrated uncertainty regions that better capture temporal correlations. A complementary line of work emphasizes the \emph{shape} and \emph{usability} of conformal prediction regions for downstream control and planning. Tumu \emph{et al.}~\cite{pmlr-v242-tumu24a} introduce optimized convex shape templates for multi-modal conformal prediction regions, enabling compact, planner-friendly uncertainty sets that can be efficiently integrated into motion planning pipelines. Conformal prediction has also been explicitly incorporated into safe planning frameworks in dynamic environments \cite{lindemann2023safe}, where prediction regions serve as probabilistic safety envelopes around predicted trajectories. More recently, Sun \emph{et al.}~\cite{sun2023conformal} combine conformal prediction with diffusion-based dynamics models to enable uncertainty-aware planning under learned stochastic dynamics. While these approaches improve calibration or the planning usability of prediction sets, they do not explicitly exploit known symmetry structure in the data.

Such symmetry structure is common in many learning problems. Real-world tasks often exhibit geometric or structural symmetries, including translation, rotation, reflection, permutation, or time-shift invariance. These symmetries are often explicitly exploited during model training via data augmentation, equivariant architectures~\cite{cohen2016group,cohen2019spherical,pmlr-v283-bousias25a,bousias2025deepequivariantmultiagentcontrol}, or, more broadly, geometric deep learning approaches~\cite{bronstein2017geometric}. When architectural equivariance is unavailable or impractical, data augmentation \cite{wang2022dataaugmentationvsequivariant} is often used as a heuristic alternative, though it provides no formal guarantees. A closely related line of work studies predictive inference under group invariance assumptions. In particular, \cite{dobriban2024symmpipredictiveinferencedata} develops a general framework for distribution-free predictive inference under arbitrary group symmetries~\cite{pillow2022symmpi,bates2023exchangeability}. Another relevant direction studies conformal prediction under geometric distribution shifts such as rotations or flips. In \cite{vanderlinden2025cp2leveraginggeometryconformal}, canonicalization is integrated into the conformal pipeline to restore approximate exchangeability and maintain coverage under geometric shifts. The primary goal in that setting is robustness: ensuring that conformal guarantees remain valid when test data differ geometrically from calibration data.

In contrast, our goal is not to restore exchangeability under shift, nor to redesign the predictive model itself, but to improve the efficiency of post-hoc uncertainty quantification by leveraging symmetry already present in the problem. This distinction is important in practice: uncertainty estimates are often constructed after training using pretrained models whose symmetry properties may be approximate or implicit rather than exact. Standard conformal prediction treats each sample in isolation, ignoring the fact that multiple transformed versions of the same input may be equivalent under a known symmetry group. As a result, symmetry-equivalent inputs may receive different nonconformity scores simply because of nuisance transformations, even though they contain the same semantic information.

In this work, we show that explicitly incorporating group symmetries into conformal prediction---without retraining the model---can substantially reduce uncertainty while preserving coverage guarantees. We introduce Equivariantized Conformal Prediction (eCP), a simple and general post-hoc procedure that infuses geometric information into CP by group-averaging the nonconformity score of a pretrained predictor. Rather than treating each data point as independent, eCP treats it as a representative of its entire group orbit, effectively redistributing nonconformity mass across symmetry-related samples.

Our approach is motivated by a key observation: while a pretrained model may not be exactly equivariant, its predictions often exhibit approximate symmetry induced by training data, architectural bias, or augmentation. By explicitly symmetrizing the nonconformity scores using group averaging, eCP reduces variability arising from arbitrary choices of coordinate frames or poses. This leads to systematically smaller calibration scores and, consequently, tighter conformal prediction sets, especially at high confidence levels where tail behavior dominates. We provide a rigorous theoretical analysis of this effect. In particular, we show that equivariantized nonconformity scores are contracted in increasing convex order relative to their unsymmetrized counterparts. This stochastic ordering implies sharper exponential-tail bounds and directly translates into improved efficiency of conformal prediction sets in expectation. Importantly, this improvement is achieved without sacrificing coverage and without assuming that the underlying predictor is exactly equivariant.

To validate our theory, we study eCP in the context of pedestrian trajectory prediction, a domain characterized by strong geometric symmetries and long-horizon uncertainty accumulation. Our experiments demonstrate that eCP consistently yields tighter prediction regions than standard conformal prediction, with the largest gains appearing at stringent confidence levels---precisely where vanilla CP is most conservative.
\noindent\textbf{Contributions}: The contributions of this paper are summarized as follows:
\begin{enumerate}
    \item We propose Equivariantized Conformal Prediction (eCP), a post-hoc method that incorporates group symmetries into conformal prediction via nonconformity score symmetrization, applicable to arbitrary pretrained prediction models.
    \item We establish that group-averaged nonconformity scores are contracted in increasing convex order, yielding improved tail behavior and sharper conformal prediction sets in expectation.
    \item We demonstrate empirically that eCP significantly reduces uncertainty in long-horizon trajectory prediction tasks, particularly at high confidence levels, while maintaining formal coverage guarantees.
\end{enumerate}
Overall, eCP provides a principled and practical mechanism for reducing conformal uncertainty by exploiting symmetry structure already present in the problem—bridging the gap between geometric inductive bias and post-hoc uncertainty quantification.

\section{Preliminaries}

\subsection{Group Theory \& Equivariant Functions}
A group $(G,\cdot)$ is a set $G$ equipped with an operator $\cdot:G\times G \rightarrow G$ that satisfies the properties of: 1) \emph{Identity:} $\exists e \in G$ such that $e \cdot g = g \cdot e = e$, 2) \emph{Associativity:} $\forall g,h,f \in G,\, g\cdot (h \cdot f) = (g \cdot h) \cdot f$, and 3) \emph{Inverse element:} $\forall g \in G,\, \exists g^{-1}$ such that $g^{-1} \cdot g = g \cdot g^{-1} = e$. Additional to its structure we can define the way that the group elements act on a space $X$ via a group action:
\begin{definition}\label{definitioin:group_action}
A map $\phi_g: X\to X$ is called an action of group element $g\in G$ on $X$ if for $e$ identity element $\phi_e(x)=x$ $\forall x\in X$ and $\phi_g\circ \phi_h=\phi_{g\cdot h}$ $\forall g,h\in G$.
\end{definition}

When $X$ is a vector space, the action of the group is defined through a linear group representation.
\begin{definition}
A linear group representation $(V,\rho)$ of a group $G$ is a map $\rho:G\to GL(V)$ from group $G$ to the general linear group $GL(V)$. The group action is then defined by the linear operator $\phi_g[x]=\rho(g)x$ for all $x\in V, g\in G$.
\end{definition}
Note here that a group action on a given space $X$ allow us to group different elements of $X$ in sets of orbits. More precisely given a group action $\phi_*$ an orbit of a element $x\in X$ is the set $\mathcal{O}_x^{\phi_*}=\left\{\phi_g(x)|g\in G\right\}$. 
\noindent In many application we require functions that respect the structure of a  group acting on their domain and codomain. 
We refer to these functions as equivariant and we formally define them as follow:
\begin{definition}
Given a group $G$ and corresponding group actions $\phi_g:X\to X$, $\psi_g:X\to X$ for $g\in G$  a function $f:X\to Y$ is said to be $(G,\phi_*,\psi_*)$-equivariant if and only if $\psi_g(f(x))=f\left(\phi_g(x)\right)\,, \forall x\in X, g\in G$.
\end{definition}
 A spacial case of equivariance is $G$-invariance occuring if $\psi_*:=id$ and the function $f$ is constant over group orbits of $X$, i.e. $f \circ \phi_* \equiv f\,,\, \forall *\in G$.

\subsection{Conformal Prediction}
Conformal prediction \cite{vovk2005algorithmic} is a framework for constructing predictive sets with finite-sample, distribution-free guarantees under minimal statistical assumptions. It provides a way to quantify uncertainty in machine learning predictions without relying on strong parametric assumptions about the data-generating process. The foundational assumption behind CP is exchangeability, which generalizes the i.i.d. assumption. A sequence of random variables $Z_{1:n}=(Z_z,...,Z_n), Z_i=(X_i,Y_i)$ is \emph{exchangeable} if
\[
P(Z_1, \ldots, Z_n) = P(Z_{\pi(1)}, \ldots, Z_{\pi(n)})
\]
for any permutation $\pi\in \mathbb{S}_n$, i.e. if its joint distribution is invariant to sample ordering. Under exchangeability, the ordering of data points carries no information, and this property enables the construction of valid p-values for candidate predictions, ensuring that calibration statistics computed on past data are valid for future predictions. This principle is what allows CP to maintain coverage guarantees even when the underlying model is misspecified or highly complex. 
We consider the setting of split conformal prediction, wherein a hold-out calibration set $\mathcal{D}_{\text{cal}}=\{(x_i,y_i)\}_{i=1:n_c}$ and a test set $\mathcal{D}_{\text{test}}=\{(x_i,y_i)\}_{i=1:n_t}$ are sampled under exchangeability for some fixed distribution \cite{angelopoulos2023gentle}. Given a pretrained prediction model $f_\theta$ and a non-conformity score function $s:\mathcal{X}\times\mathcal{Y}\rightarrow \mathbb{R}$ measuring how atypical a sample is by encoding the disparity between prediction and label, we compute calibration scores $s_i = s\big(f_\theta(x_i), y_i \big)$ on held-out $\mathcal{D}_{\text{cal}}$. The choice of nonconformity function strongly influences the efficiency/tightness of prediction sets but not their validity. For a new sample $(x_{n_c+1},y_{n_c+1})$, define the prediction set
\[
C_{1-\alpha}(x_{n_c+1}) = \{ y\in \mathcal{Y} : s\big(f_\theta(x_{n_c+1}), y_i \big) \le q_{1-\alpha} \},
\]
where $q_{1-\alpha}$ is the $(1-\alpha)$-quantile of the empirical calibration non-conformity distribution $S=\{s_{1:n_c}\}$ for tolerated miscoverage rate $\alpha\in(0,1)$. Then, with high probability a valid coverage guarantee on inclusion of the true label $y_{n_c+1}$ stands
\[
P\big(y_{n_c+1} \in C_{1-\alpha}(x_{n_c+1}) \big) \ge 1 - \alpha,
\]
under exchangeability \cite{lei2018distribution,romano2019conformalized}. While validity is unconditional, efficiency — the expected size or tightness of the prediction set — depends on the informativeness of the model and the nonconformity function. In practice, a well-calibrated and expressive base model yields smaller, more informative sets.

\section{Symmetries-infused Conformal Prediction}\label{section:Symmetries-infused Conformal Prediction}

Let $(\mathcal{X},\mathcal{Y})$ be measurable spaces and $G$ a (finite or compact) group
acting measurably on $\mathcal{X}$ and $\mathcal{Y}$.
Let $F\subset \mathcal{M}_b(\mathcal{X})$ be the set of bounded measurable functions $f:\mathcal{X}\rightarrow \mathcal{Y}$ and define its subset of $G$-equivariant functions $F_G:=\{f\in F \,:\, f\circ\phi_g = \psi_g \circ f\,,\, \forall g\in G\}$. A probability measure $\mu\in \mathcal{P}(\mathcal{Z}=\mathcal{X}\times \mathcal{Y})$ is $G$-invariant if, for $\tau_g:=(\phi_g,\psi_g)$, $\mu=\mu \circ \tau_g\,,\, \forall g\in G$.

\begin{definition}[Group-invariant distribution]
    Let \(Z=(X,Y)\) be a random variable with a probability distribution \(P\) on a sample space \(\mathcal{Z}=\mathcal{X}\times \mathcal{Y}\). The distribution \(P\) is group invariant if \(P\big(\tau_g(A)\big)=P(A)\) for all \(g\in G\) and any measurable subset \(A\subseteq \mathcal{Z}\).
\end{definition}
This is equivalent to saying that the joint probability measure is $G$-invariant and $(X,Y)\overset{d}{=}\big(\phi_g(X),\psi_g(Y)\big)$. 
These structures generalize permutation invariance, i.e. exchangeability, to more general geometric or combinatorial symmetries.

\begin{definition}
A sequence of random variables $(Z_1, \ldots, Z_{n+1})$ is \emph{$G^{n+1}$-exchangeable} if
\((Z_1, \ldots, Z_{n+1}) \stackrel{d}{=} \big(\tau_{g_1}(Z_1), \ldots, \tau_{g_{n+1}}(Z_{n+1})\big)\,,\, \forall g_1,\ldots,g_{n+1} \in G\)
\end{definition}
The $G^{n+1}$-exchangeability generalizes the exchangeability requirement of CP from just the permutation group $S_{n+1}$, to include the geometric transformations of samples by $G^{n+1}$, implying that the joint distribution is $G^{n+1}$-invariant.

\begin{assumption}\label{as:G-invariant_score} 
The nonconformity score function $s:F_G(\mathcal{X})\times\mathcal{Y}\to\mathbb{R}$ is \emph{$G$-invariant} in the sense that $s\big(\psi_g(y_i),\psi_g(y_j)\big)=s(y_i,y_j)\,,\, \forall g\in G$.
\end{assumption}

\begin{definition}[Symmetrization Operator]\label{def:symmetrization_operator_invariance}
    For some $G$-invariant non-conformity score $s:F_G(\mathcal{X})\times \mathcal{Y}\rightarrow \mathbb{R}_{\geq0}$, define the score symmetrization operator $\Pi:(F\times\mathcal{X})\times \mathcal{Y}\rightarrow \mathbb{R}_{\geq0}$:
    \begin{align}
    \Pi_G[s;f](x,y) := \int_{G} s\Big( f\big(\phi_{g^{-1}}(x) \big),\psi_{g^{-1}}(y) \Big) d\mu(g) \nonumber \\= \mathbb{E}_{\mu_G} \Big[ s\Big(f\big( \phi_{g^{-1}}(x)\big),\psi_{g^{-1}}(y)  \Big) \Big]
        \label{eq:symmetrization_operator_invariance}
    \end{align}
    where $\mu_G$ is the unique, left-invariant Haar probability measure on the compact Hausdorff topological group $G$, with $\mu_G(G)=1$.
\end{definition}
\begin{lemma}\label{lem:PiG_invariance_projection}
Under Assumption~\ref{as:G-invariant_score}, the symmetrization operator \ref{eq:symmetrization_operator_invariance} is a $G$-invariant projection in the sense that
\[
\Pi_G[s;f]\big(\phi_h(x),\psi_h(y)\big)=\Pi_G[s;f](x,y)\;\;\;\forall h\in G,\ (x,y)\in X\times Y
\]
Moreover, if $f\in F_G$ then $\Pi_G[s;f](x,y)=s\!\big(f(x),y\big)\, \forall (x,y)$, i.e. idempotence on $G$-equivariant models.
\end{lemma}

\begin{lemma}\label{lem:expectation_orbit_conditional}
    The symmetrized non-conformity score is the expected non-conformity score across the orbit induced by $G$, i.e. $\Pi_G[s;f](x,y) = \mathbb{E}_{(\mathcal{X}\times \mathcal{Y})\sim P}[s(f(X),Y) \,|\, (X,Y)\in\mathcal{O}^{(\phi,\psi)}_{(x,y)} ]$.
\end{lemma}

\begin{cor}
    From the law of total probability, the mean of the score symmetrization operator is the same as that of the score function itself, i.e. $\mathbb{E}_{\mathcal{X}\times \mathcal{Y}\sim P}\big[\Pi_G[s;f](X,Y)\big]= \mathbb{E}_{\mathcal{X}\times \mathcal{Y}\sim P}[s(f(X),Y)]$.
\end{cor}
With the aforementioned structures, a coverage certificate may be provided for the finite-sample validity of split Conformal Prediction for symmetry-preserving distributions, i.e. the calibration samples and the test sample $(X_{n_c+1},Y_{n_c+1})$ are $G^{n_c+1}\times\mathbb{S}_{n_c+1}$-exchangeable.
\begin{theorem}\label{thm:split-g-equivariant-validity}
Consider $\mathcal{D}_{\text{train}}$ training set and $\mathcal{D}_{\text{cal}}=\{(X_i,Y_i)\}_{i=1:n_c}$ calibration set, both drawn from a $G\times S_{N}$-invariant distribution $P$ on $\mathcal{X}\times\mathcal{Y}$, and let Assumption \ref{as:G-invariant_score} stand. 
Define group invariant calibration scores $\widetilde{S}_i=\Pi[s;f](X_i,Y_i)$ for $(X_i,Y_i)\in\mathcal{D}_{\text{cal}}$ and, for any test sample $x_{n_c+1}\in \mathcal{X}$ and candidate label $y\in\mathcal{Y}$, define the test score $\widetilde{S}_{n_c+1}(y)=\Pi[s;f](X_{n_c+1},y)$.
Let $\widetilde{S}_{(1)}\le\cdots\le \widetilde{S}_{(m)}$ denote the ordered calibration scores and, for some miscoverage rate $\alpha\in(0,1)$, let
\(k=\left\lceil (m+1)(1-\alpha)\right\rceil,\;\; q=\widetilde{S}_{(k)}.\)
Then, for equivariant split conformal prediction set $C_{1-\alpha}^{(G)}(X_{n_c+1})=\{y\in\mathcal{Y}:\widetilde{S}_{n_c+1}(y)\le q\}$ it stands that $$P\big(Y_{n_c+1}\in C_{1-\alpha}^{(G)}(X_{n_c+1})\mid \mathcal{D}_{\text{train}}\big)\ge 1-\alpha$$
\end{theorem}
\begin{cor}\label{cor:variance_decomp_symm }
    Consider a $G$-invariant distribution $P$ on $\mathcal{X}\times\mathcal{Y}$, and let Assumption \ref{as:G-invariant_score} stand. Then, for $(X,Y)\sim P$, it stands that $\mathrm{Var}\big(s(f(X),Y)\big) \geq \mathrm{Var}\!\big(\Pi_G[s;f](X,Y)\big)$.
\end{cor}
Variance reduction under symmetrized non-conformity scores, however, only concentrates the mass of the distribution of the calibration non-conformity scores; it does not shrink the tail events that determine the  conformal set $C_{1-\alpha}^{(G)}$. The latter is solely governed by the quantiles of the distribution.
\begin{assumption}\label{as:convex_score}
    The non-conformity score function $s:(\mathcal{X}\times F)\times \mathcal{Y}\rightarrow \mathbb{R}_{\geq0}$ is convex in the prediction argument. 
\end{assumption}
For a pre-trained model $f_\theta:\mathcal{X}\rightarrow \mathcal{Y}$ on $\mathcal{D}_\text{train}$, Assumptions \ref{as:G-invariant_score} and \ref{as:convex_score} on a non-conformity score function yield via Jensen's inequality:
\begin{align}
    \Pi_G[s;f](x,y) &:= \int_{G} s\Big( f_\theta\big(\phi_{g^{-1}}(x) \big),\psi_{g^{-1}}(y) \Big) d\mu(g) \nonumber\\
    &= \int_{G} s\Big( \psi_g \big(f_\theta\big(\phi_{g^{-1}}(x) \big)\big),y \Big) d\mu_G(g) \\
    &\geq s\bigg( \int_G \psi_g \circ f_\theta \circ \phi_{g^{-1}}(x) d\mu_G(g) , y \bigg)\label{eq:jENSEN_INEQUALITY}
\end{align}
We can now define the equivariantized pre-trained model:
\begin{align}\label{eq:equivariantized_predictor}
    f_\theta^G(x):=\int_G \psi_g \circ f_\theta \circ \phi_{g^{-1}}(x)\,d\mu_G(g)\\=\mathbb{E}_G \big[ \psi_g \circ f_\theta \circ \phi_{g^{-1}}(x) \big]
\end{align}
This is the canonical projection of $f$ onto the space of $G$-equivariant maps. The domain of the non-conformity score is, thus, restricted now to $s:(\mathcal{X}\times F_G)\times \mathcal{Y}\rightarrow \mathbb{R}_{\geq0}$. The proof that predictor $f_\theta^G$ is $G$-equivariant is similar to that of Lemma \ref{lem:PiG_invariance_projection} and is therefore omitted. From Assumption \ref{as:G-invariant_score} it is obvious that the non-conformity score remain $G$-invariant for $G$-equivariant models. Furthermore, Theorem \ref{thm:split-g-equivariant-validity} still stands, allowing for a more practical implementation of the Equivariant Split Conformal Prediction. Combining \ref{eq:jENSEN_INEQUALITY} with Lemma \ref{lem:expectation_orbit_conditional}
\begin{align}\label{eq:inequality_}
    &s\bigg( \mathbb{E}_G \big[ \psi_g \circ f_\theta \circ \phi_{g^{-1}}(x) \big] , y \bigg) \nonumber\\&\qquad\leq \mathbb{E}_{(\mathcal{X}\times \mathcal{Y})\sim P}[s(f_\theta(X),Y) \,|\, (X,Y)\in\mathcal{O}^{(\phi,\psi)}_{(x,y)} ]
\end{align}
and taking expectations yields
\begin{align}\label{eq:equivariantized_mean_score_decrease}
     \mathbb{E}_{(\mathcal{X}\times \mathcal{Y})\sim P}[s(f^G(X),Y)]\leq \mathbb{E}_{(\mathcal{X}\times \mathcal{Y})\sim P}[s(f(X),Y)]
\end{align}
i.e. the non-conformity scores are expected to contract when we equivariantize the pretrained model according to \ref{eq:equivariantized_predictor}. This is owed to the infusion of geometric information at every prediction.

\begin{algorithm}[t]
\caption{eCP: Equivariant Conformal Prediction}
\label{alg:ecp-approx-alg2e}
\DontPrintSemicolon
\KwIn{$\mathcal{D}_{\mathrm{cal}}=\{(x_i,y_i)\}_{i=1}^{n_c}$, $f_\theta:\mathcal{X}\to\mathcal{Y}$, $G$, $s$, $\alpha\in(0,1)$, $m$}
\KwOut{Conformal prediction set $\widehat C^{(G)}_{1-\alpha}(x)$}

\If{$G$ is finite and enumerable}{
    $m \gets |G|\;,\;$
    $\Xi_m \gets \{g_i | \forall i \in \{1,\dots,|G|\}\}\;$
}
\Else{
    $\Xi_m \gets\{ g_1,\dots,g_m\} \sim \mu_G$
}

\For{$i\gets 1$ \KwTo $n_c$}{
    $\widehat y_i \gets \frac{1}{m}\sum_{g_j\in \Xi_m}\psi_{g_j}\!\bigl(f_\theta(\phi_{g_j^{-1}}(x_i))\bigr)$\;
    
    $S_i \gets s\!\left(\widehat y_i, y_i\right)$\;
}

Sort $\{S_i\}_{i=1}^{n_c}$ as $S_{(1)}\le \cdots \le S_{(n_c)}$

$k \gets \left\lceil (n_c+1)(1-\alpha)\right\rceil$

$q \gets S_{(k)}$\;

\Return{$\widehat C^{(G)}_{1-\alpha}(x)=\left\{y\in\mathcal{Y}: s\!\left( f_\theta^{G}(x),y\right)\le q\right\}$}\;
\end{algorithm}

\subsection{Expected quantile contraction via symmetrized predictors}
\begin{definition}[Increasing convex order of distributions]
For integrable random variables $U,V$, we write $U\preceq_{\mathrm{icx}} V$ if
$\mathbb{E}[\nu(U)]\le \mathbb{E}[\nu(V)]$ for all \emph{increasing convex} $\nu:\mathbb{R}\to\mathbb{R}$.
Equivalently, the stop-loss ordering is defined as
\[
U\preceq_{\mathrm{icx}} V \;\Longleftrightarrow\;
\mathbb{E}\big[(U-t)_+\big] \le \mathbb{E}\big[(V-t)_+\big]\quad \forall\,t\in\mathbb{R}.
\]
\end{definition}

\begin{theorem}\label{thm:icx}
Consider $G^{n_c+1}\times\mathbb{S}_{n_c+1}$-exchangeability and convexity of the $G$-invariant non-conformity score $s$ in its prediction argument. The non-conformity distributions of the pretrained model and its equivariantized version, then, satisfy
\begin{align*}
    s\big(f^G(X),Y\big) = S_{f^G}\preceq_{\mathrm{icx}} S_f= s\big(f(X),Y\big) \;\;\;,\; (X,Y)\sim P
\end{align*}
\end{theorem}

This means that \(Y\) is "more spread out" than \(X\) in a specific sense, where the "more spread out" property is captured by the behavior of increasing convex functions. If $f_\theta$ is already $G$-equivariant (so $f_\theta^G=f_\theta$) or if the non-conformity score is affine along the convex hull of the $G$-orbit of the predictions (so the Jensen step is tight), then equality holds for all increasing convex functions.

For any $p\ge 1$ with $\mathbb{E}[S_f^p]<\infty$, taking $\nu(t)=t^p$ gives $\mathbb{E}\big[S_{f^G}^p\big]\;\le\;\mathbb{E}\big[S_f^p\big]$.
In particular, $\mathbb{E}[S_{f^G}]\le \mathbb{E}[S_f]$ and $\mathbb{E}[S_{f^G}^2]\le \mathbb{E}[S_f^2]$. Contrary to increasing convex ordering, standard convex order implies the equality of means, i.e., if $X \leq_{cx} Y$, then $\mathbb{E}[X] = \mathbb{E}[Y]$ (use $\nu(x)=\pm x$), and so by taking $\nu(x) = x^2$, we obtain that if $X \leq_{cx} Y$, then $\mathrm{Var}[X] \leq \mathrm{Var}[Y]$. Also, in case $\mathbb{E}[X] = \mathbb{E}[Y]$, then $X \leq_{cx} Y \Leftrightarrow X \leq_{icx} Y$.

\begin{theorem}\label{theorem:cvar_ordering}[CVaR contraction under symmetrization]
Fix $\alpha\in[0,1)$ and assume $\mathbb{E}[S_f^+]<\infty$. 
Under the assumptions of Theorem~\ref{thm:icx}, with
$S_f:=s(f_\theta(X),Y)$ and $S_{f^G}:=s(f_\theta^G(X),Y)$, $\forall\alpha\in(0,1)$ it stands that:
\[
0 \;\le\; \mathrm{CVaR}_\alpha(S_f)- \mathrm{CVaR}_\alpha(S_{f^G}) \;\le\; \dfrac{\mathbb{E}[S_f - S_{f^G}]}{1-\alpha}
\]
\end{theorem}
Unfortunately, it is impossible to order $\mathrm{VaR}_\alpha$, which is equivalent to point-wise ordering of the quantiles, without further assumptions on the distributions to attain stochastic dominance. However, Theorem \ref{theorem:cvar_ordering} allows for ordering the tail masses of the distributions, or equivalently:
\begin{lemma}\label{lemma:2.2}
Let $S_f$ and $S_{f^G}$ be two random variables with cumulative distribution functions 
$F_f$ and $F_{f^G}$, and quantile functions $F_{f}^{-1}$ and $F_{f^G}^{-1}$, respectively. If $S_{f^G}\preceq_{\mathrm{icx}} S_{f}$, then:
\[
\int_p^1 F_{f^G}^{-1}(t)\,dt \leq \int_p^1 F_{f}^{-1}(t)\,dt \quad \forall p \in (0,1).
\]
\end{lemma}
\begin{cor}\label{cor:expetced_quantile_contraction}
For $S_{f^G}\preceq_{\mathrm{icx}} S_{f}$, the calibration quantiles of the symmetrized predictor contract in expectation, i.e.
\[
\mathbb{E}\!\left[F^{-1}_{S_{f^G}}(U)\right]
\;\le\;
\mathbb{E}\!\left[F^{-1}_{S_f}(U)\right],\quad U\sim \mathrm{Unif}(0,1)
\]
If $S_f, S_{f^G}$ are continuous random variables with interval support, for any increasing convex function $\nu$, it stands that:
\[
\mathbb{E}\left[ \nu\big(F_{f^G}^{-1} \circ F_{f}(S_f)\big) \,\middle|\, S_f > t \right]
   \leq 
   \mathbb{E}\left[ \nu(S_f) \,\middle|\, S_f > t \right],
   \quad \forall t \in \mathbb{R}
\]
\end{cor}

\section{Expected Shrinkage of Conformal Sets}
Consider $f_\theta:\mathcal{X}\rightarrow \mathcal{Y}$ trained on $\mathcal{D}_{\text{train}}$ and $\mathcal{D}_{\text{cal}}=\{(X_i,Y_i)\}_{i=1:n_c}$ calibration set drawn from a $G$-invariant distribution under exchangeability, and consider positive semi-definite, $G$-invariant, convex non-conformity score. From Split Conformal Prediction theory, for a test sample $(X_{n_c+1},Y_{n_c+1})$ it stands that $$P\big(Y_{n_c+1}\in C_{1-\alpha}^{f_\theta}(X_{n_c+1})\big)\ge 1-\alpha$$
for prediction set $$C_{1-\alpha}^{f_\theta}(X_{n_c+1})=\{y\in\mathcal{Y}:s(f_\theta(X_{n_c+1}),Y_{n_c+1})\le F^{-1}_{S_{f_\theta}}(1-\alpha)\}$$ where $F^{-1}_{S_{f_\theta}}(1-\alpha):=\{p\in\mathbb{R}\mid F_{S_{f_\theta}}(x)\geq 1-\alpha\}$ the continuous and differentiable quantile function on the CDF $F_{S_{f_\theta}}$ from $\mathcal{D}_{\text{cal}}$.
From theorem \ref{thm:split-g-equivariant-validity} on Finite-Sample Validity of Equivariantized Split Conformal Prediction, the same probabilistic bound stands for the equivariantized predictor, but the conformal prediction set becomes $$C_{1-\alpha}^{f_\theta^G}(X_{n_c+1})=\{y\in\mathcal{Y}:s(f_\theta^G(X_{n_c+1}),Y_{n_c+1})\le F^{-1}_{S_{f_\theta^G}}(1-\alpha)\}$$

In this section we motivate the use of eCP compared to classic CP by showing that the expected volume of the conformal sets shrinks for the same miscoverage rate $\alpha\in(0,1)$, i.e. $\mathbb{E}[\mathrm{Vol}(C_{1-\alpha}^{f_\theta^G})] \leq \mathbb{E}[\mathrm{Vol}(C_{1-\alpha}^{f_\theta})]$.

\subsection{From quantile to volume contraction of conformal set}



Assuming that $s(f(x),y)=\|f(x)-y\|$ with a Euclidean or norm-induced geometry, $C_{1-\alpha}^{f_\theta}(x)$ is a ball (or convex set) centered at $f_\theta(x)$ with radius $F^{-1}_{S_{f_\theta}}(1-\alpha)$ of volume $\mathrm{Vol}(C_{1-\alpha}^{f_\theta}) = \kappa  F^{-1}_{S_{f_\theta}}(1-\alpha)^d$ and $\mathrm{Vol}(C_{1-\alpha}^{f_\theta^G}) = \kappa F^{-1}_{S_{f_\theta^G}}(1-\alpha)^d$. Let $p\sim U(1-\alpha,1)$. 
From the Mean Value Theorem $${F^{-1}_{S_{f_\theta}}(p)}^d -{F^{-1}_{S_{f_\theta^G}}(p)}^d = d q(p)^{d-1}({F^{-1}_{S_{f_\theta}}(p)}-{F^{-1}_{S_{f_\theta^G}}(p)})$$
for $q(p)\in\Big[ \min\{{F^{-1}_{S_{f_\theta}}(p)} ,{F^{-1}_{S_{f_\theta^G}}(p)}\}, \max\{{F^{-1}_{S_{f_\theta}}(p)} ,{F^{-1}_{S_{f_\theta^G}}(p)}\} \Big]$ and $\forall p \in (1-\alpha,1)$ it stands that 
\begin{align*}
    q_0:=&\min\{{F^{-1}_{S_{f_\theta}}(1-\alpha)} ,{F^{-1}_{S_{f_\theta^G}}(1-\alpha)}\} \\
    &\quad\leq q(p) \leq \max\{{F^{-1}_{S_{f_\theta}}(1)} ,{F^{-1}_{S_{f_\theta^G}}(1)}\}=:q_1
\end{align*}
with $q_0\geq 0$. Then, the expected shrinkage of the conformal regions for the upper quantiles is
\begin{align}
    D\mathrm{Vol}_{1-\alpha}&:=\mathbb{E}_{p\sim U(1-\alpha,1)}[\mathrm{Vol}(C_{p}^{f_\theta}) -\mathrm{Vol}(C_{p}^{f_\theta^G})] \nonumber\\
     &= \kappa \,\mathbb{E}_{p\sim U(1-\alpha,1)}\Big[{F^{-1}_{S_{f_\theta}}(p)}^d -{F^{-1}_{S_{f_\theta^G}}(p)}^d \Big] \nonumber\\
    &= d \,\kappa\, \mathbb{E}_{p\sim U(1-\alpha,1)}\Big[q(p)^{d-1}\big({F^{-1}_{S_{f_\theta}}(p)} -{F^{-1}_{S_{f_\theta^G}}(p)} \big)\Big] \nonumber\\
    \Rightarrow d \,q_0^{d-1}\,&\kappa\, \mathbb{E}_{p\sim U(1-\alpha,1)}\Big[{F^{-1}_{S_{f_\theta}}(p)} -{F^{-1}_{S_{f_\theta^G}}(p)} \Big] \leq D\mathrm{Vol}_{1-\alpha} \nonumber\\
    &\leq d \,q_1^{d-1}\,\kappa\, \mathbb{E}_{p\sim U(1-\alpha,1)}\Big[{F^{-1}_{S_{f_\theta}}(p)} -{F^{-1}_{S_{f_\theta^G}}(p)} \Big] 
    \label{eq:bound_exp_vol_1}
\end{align}
Observe that the upper quantiles can be restated as $\mathrm{CVaR}_\alpha(S_{f_\theta}) = \mathbb{E}_{p\sim U(0,1)}[F^{-1}_{S_{f_\theta}}(p)\,|\,p\geq 1-\alpha] $ and $\mathrm{CVaR}_\alpha(S_{f_\theta^G}) = \mathbb{E}_{p\sim U(0,1)}[F^{-1}_{S_{f_\theta^G}}(p)\,|\,p\geq 1-\alpha]$, and \ref{eq:bound_exp_vol_1} becomes
\begin{align}
    d \,q_0^{d-1}\,\kappa\, \big[\mathrm{CVaR}_\alpha(S_{f_\theta})- \mathrm{CVaR}_\alpha(S_{f_\theta^G})\big] \leq D\mathrm{Vol}_{1-\alpha} \nonumber\\ \leq d \,q_1^{d-1}\,\kappa\, \big[\mathrm{CVaR}_\alpha(S_{f_\theta})- \mathrm{CVaR}_\alpha(S_{f_\theta^G})\big]
    \label{eq:bound_exp_vol_2}
\end{align}
Applying Theorem \ref{theorem:cvar_ordering} in \ref{eq:bound_exp_vol_2} yields $\forall\alpha\in(0,1)$
\begin{align}
    0 &\leq \mathbb{E}_{p\sim U(1-\alpha,1)}[\mathrm{Vol}(C_{p}^{f_\theta}) -\mathrm{Vol}(C_{p}^{f_\theta^G})] \nonumber\\ &\leq \dfrac{d\,\kappa}{1-\alpha} \max\{S_{f_\theta},S_{f_\theta^G}\}^{d-1}\, \mathbb{E}[S_{f_\theta}-S_{f_\theta^G}]
\end{align}

\section{\textsc{Experiments}}
\begin{figure*}[h!]
\centering

\begin{subfigure}[t]{0.5\linewidth}
  \centering
  \includegraphics[width=\linewidth]{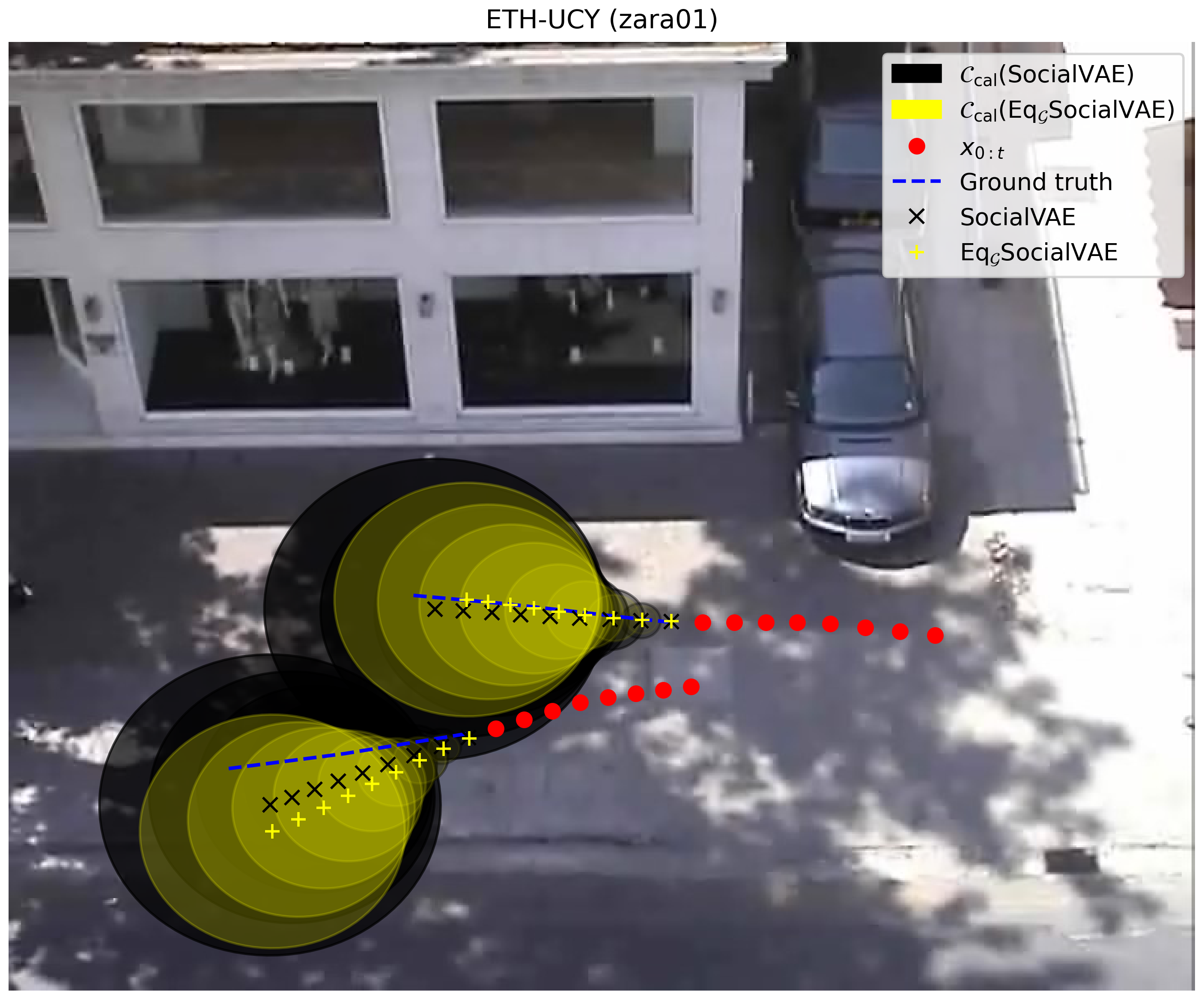}
\end{subfigure}\hfill
\begin{subfigure}[t]{0.5\linewidth}
  \centering
  \includegraphics[width=\linewidth]{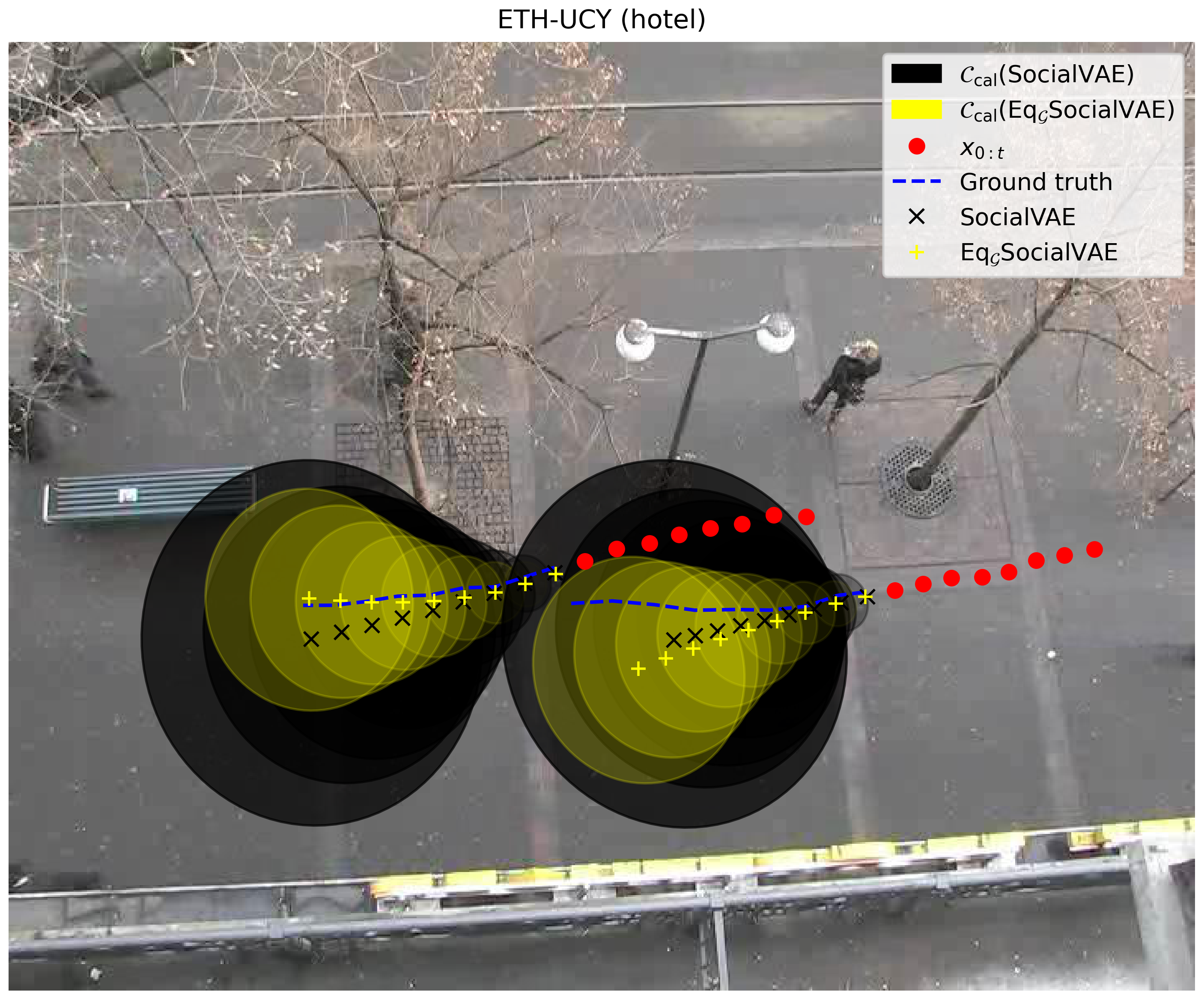}
\end{subfigure}\hfill
\begin{subfigure}[t]{0.5\linewidth}
  \centering
  \includegraphics[width=\linewidth]{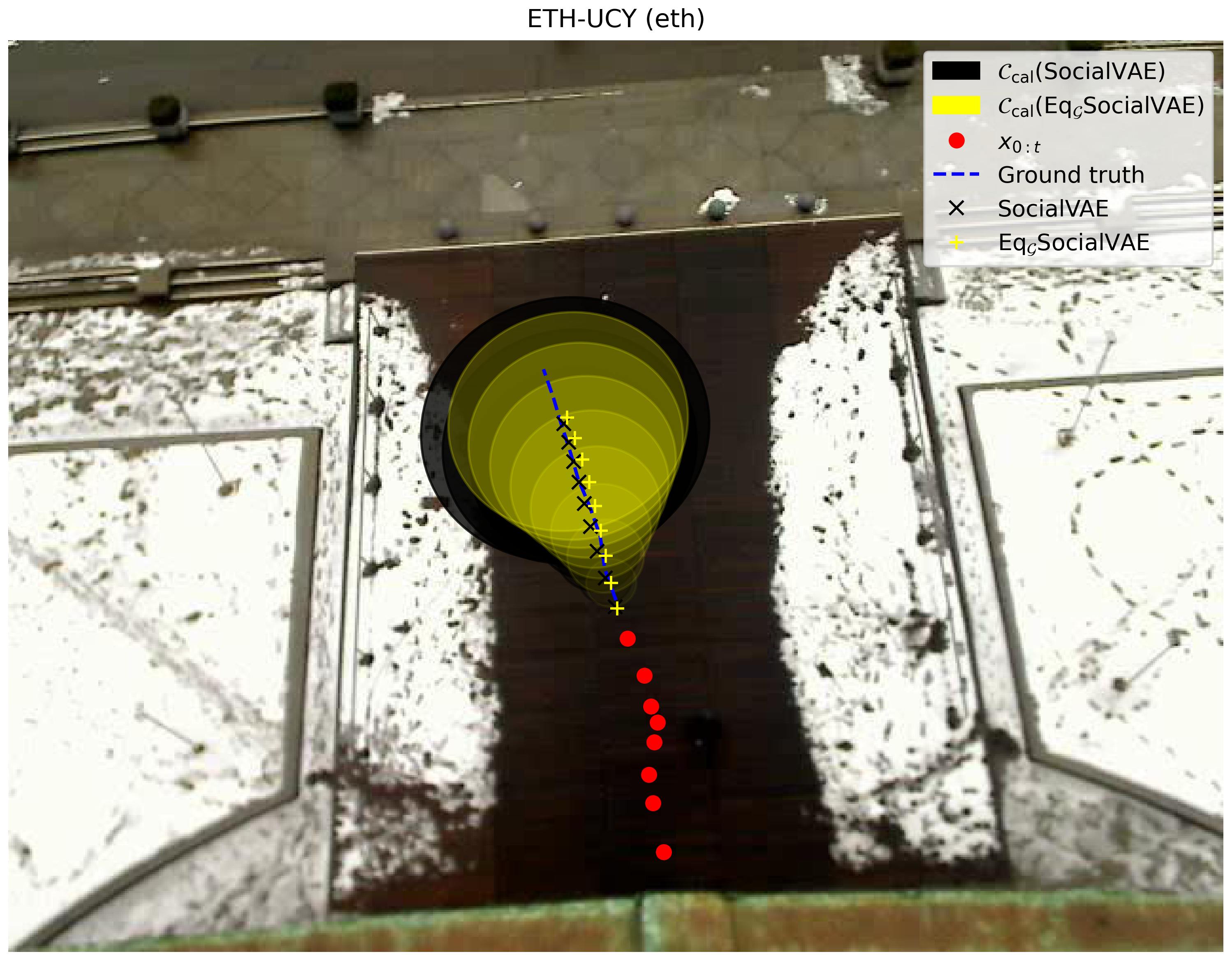}
\end{subfigure}\hfill
\begin{subfigure}[t]{0.5\linewidth}
  \centering
  \includegraphics[width=\linewidth]{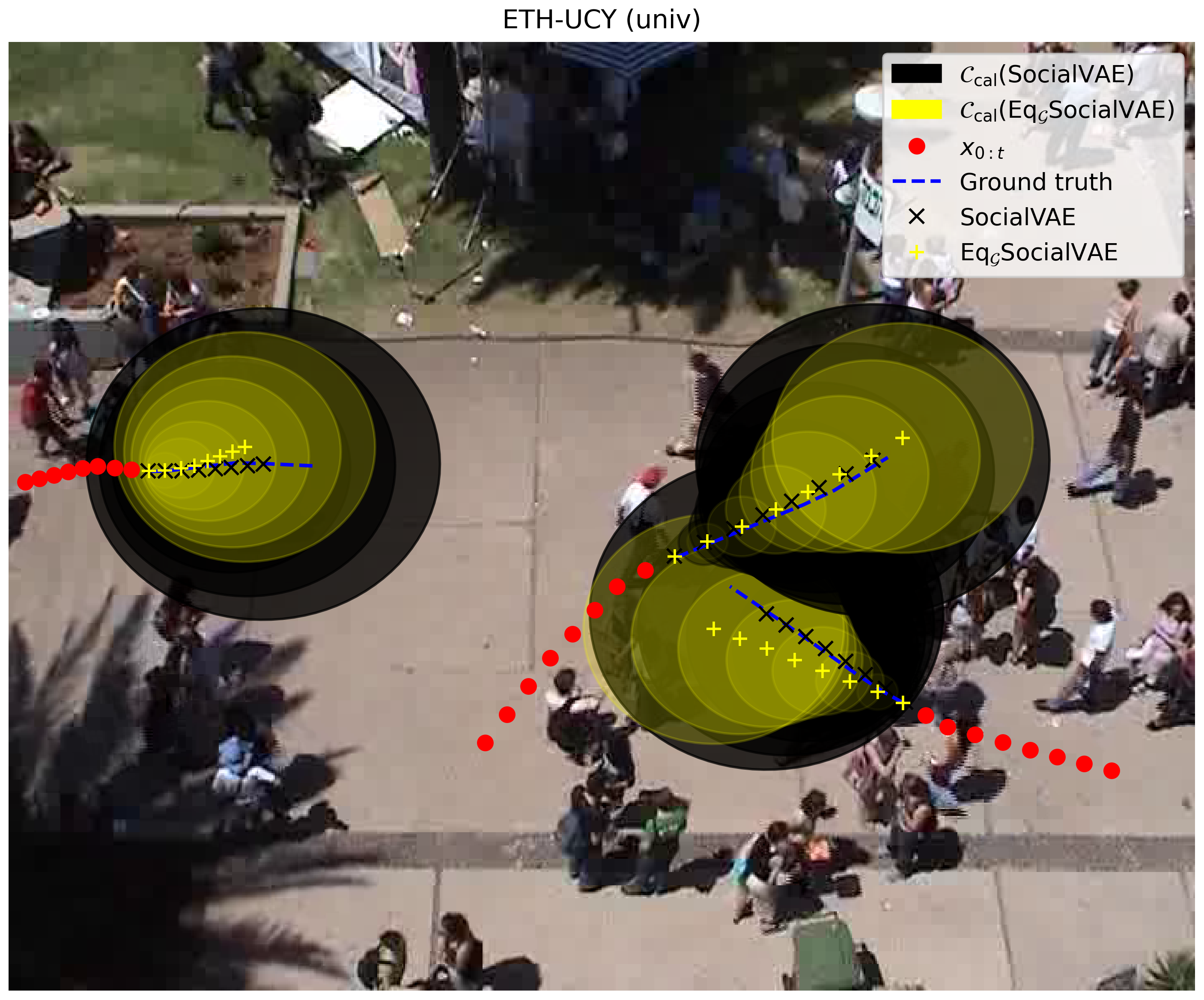}
\end{subfigure}\hfill
\begin{subfigure}[t]{0.33\linewidth}
  \centering
  \includegraphics[width=\linewidth]{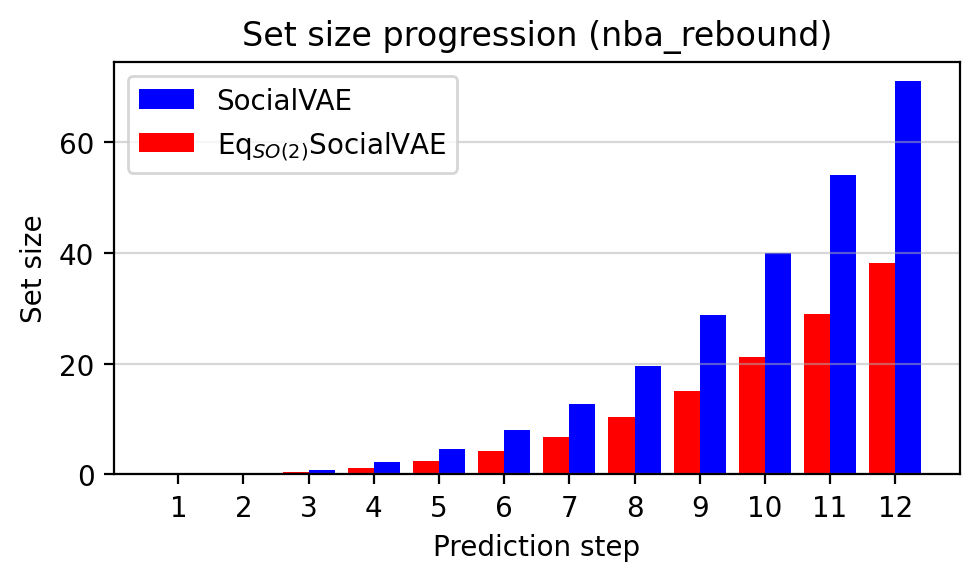}
\end{subfigure}\hfill
\begin{subfigure}[t]{0.33\linewidth}
  \centering
  \includegraphics[width=\linewidth]{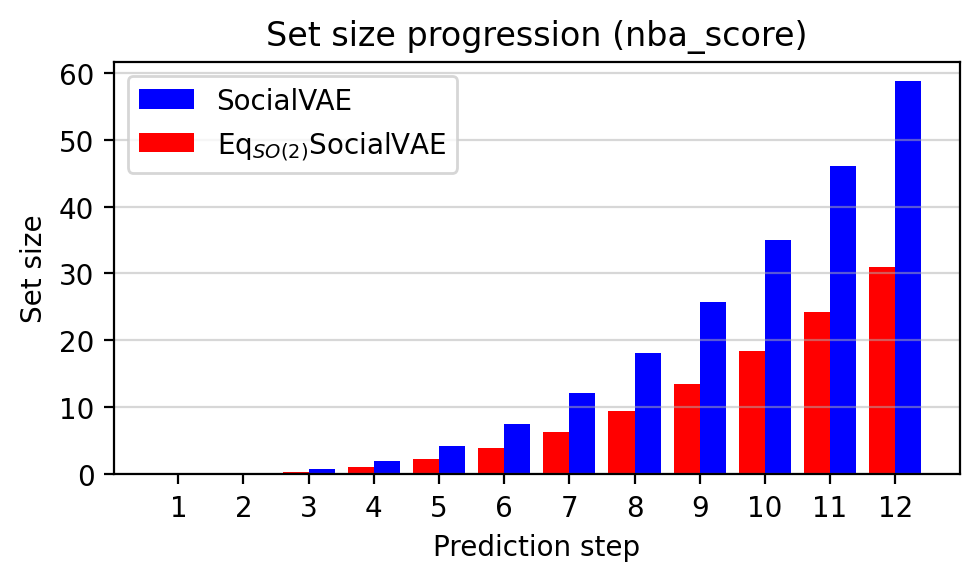}
\end{subfigure}
\begin{subfigure}[t]{0.33\linewidth}
  \centering
  \includegraphics[width=\linewidth]{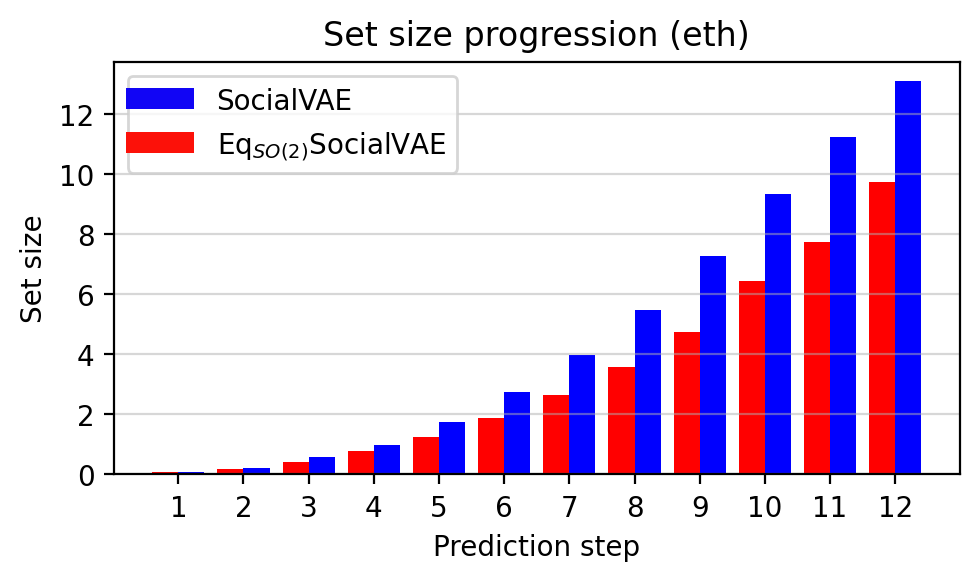}
\end{subfigure}
\caption{Set size reduction via symmetrization of Conformal prediction in multi-step prediction in ETH-UCY and NBA Rebound/Score datasets.}
\label{fig:set_size}
\end{figure*}

We evaluate Equivariantized Conformal Prediction (eCP) on long-horizon pedestrian trajectory prediction, a domain characterized by strong geometric symmetries and rapidly accumulating uncertainty. The experiments are designed to validate the theoretical results of Section~4, namely that equivariantization contracts the distribution of nonconformity scores and yields tighter conformal prediction sets while preserving finite-sample coverage guarantees—especially at high confidence levels. We conduct experiments on standard pedestrian trajectory prediction benchmarks: ETH--UCY~\cite{pellegrini2009youll}, the Stanford Drone Dataset~\cite{robicquet2016learning} and the SportVU NBA movement dataset. These datasets are widely used to evaluate uncertainty in multi-agent, long-horizon forecasting and exhibit planar rotational symmetries. Following standard protocols, each input consists of an observed pedestrian trajectory of 8 timesteps, and the task is to predict future positions over a 12-step horizon. Uncertainty is quantified by constructing conformal prediction sets over future trajectories. 
\begin{figure*}[h!]
    \centering
    \includegraphics[width=\linewidth]{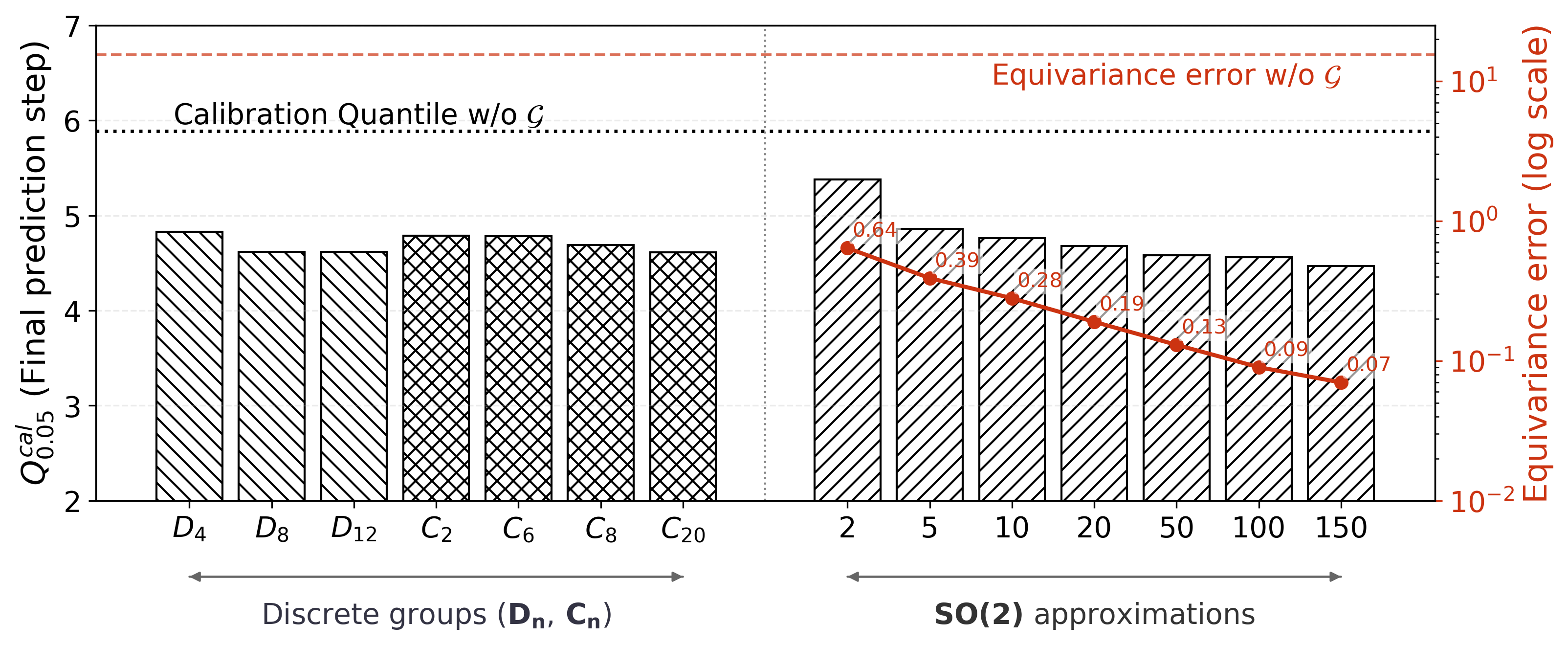}
    \caption{Ablation studies on the group size and approximate equivariance.}
    \label{fig:group_ablation}
\end{figure*}
We evaluate eCP on top of two pretrained trajectory predictors with fundamentally different modeling assumptions: 1) \textbf{SocialVAE}~\cite{socialvae2022}: a stochastic latent-variable model for multimodal trajectory prediction, 2)\textbf{TUTR}~\cite{zhao2021tutr}: a transformer-based deterministic trajectory predictor. No retraining or fine-tuning is performed. All symmetry injection is applied strictly \emph{post hoc}, demonstrating the model-agnostic nature of eCP. We consider planar rotation symmetry groups of increasing richness $\mathcal{G} \in \{C_4, C_8, \mathrm{SO}(2)\}$. 
Since the orbit-averaging integral in~\eqref{eq:equivariantized_predictor} is generally intractable, we implement model equivariantization by numerical group averaging. For finite groups $G$, we compute the exact discrete average
\[
f_\theta^G(x)=\frac{1}{|G|}\sum_{g\in G}\psi_g\!\bigl(f_\theta(\phi_{g^{-1}}(x))\bigr).
\]
For continuous groups, we approximate the Haar integral by Monte Carlo sampling:
\[
f_\theta^{G}(x) \approx \frac{1}{m}\sum_{j=1}^m \psi_{g_j}\!\bigl(f_\theta(\phi_{g_j^{-1}}(x))\bigr),
\qquad g_j \overset{i.i.d.}{\sim} \mu_G.
\]
We then apply conformal calibration using the approximate equivariantized predictor $f_\theta^{G,m}$. The resulting methods are denoted EqC$_4$, EqC$_8$, and EqSO$(2)$, respectively. We employ split conformal prediction with Euclidean displacement error as the nonconformity score. Calibration is performed on a held-out set drawn under exchangeability assumptions, and all results are averaged over 15 random calibration splits. We report two primary metrics that directly capture the efficiency--validity trade-off predicted by theory: 1) \textbf{Calibration Quantile $Q_{0.05}$}: the empirical 95\% conformal radius (lower is better), 2) \textbf{Empirical Coverage (Cov$_{95\%}$)}: the fraction of test trajectories contained in the conformal prediction set (target: 95\%).
\begin{table*}[h!]
\centering
\setlength{\tabcolsep}{6pt}
\renewcommand{\arraystretch}{1.15}


\begin{adjustbox}{width=\textwidth}
\begin{tabular}{lcccccc}
\toprule
& \multicolumn{2}{c}{\textbf{ETH}}
& \multicolumn{2}{c}{\textbf{SDD}}
& \multicolumn{2}{c}{\textbf{HOTEL}} \\
\cmidrule(lr){2-3}\cmidrule(lr){4-5}\cmidrule(lr){6-7}
\textbf{Method}
& $Q_{0.05}\downarrow$ & Cov$_{95\%}\uparrow$
& $Q_{0.05}\downarrow$ & Cov$_{95\%}\uparrow$
& $Q_{0.05}\downarrow$ & Cov$_{95\%}\uparrow$ \\
\midrule

SocialVAE & 3.47±0.05 & 94.94±0.36 & 5.19±0.43 & 94.76±1.52 & 3.39±0.05 & 94.99±0.34 \\
Eq$_{\text{C4}}$SocialVAE & 2.77±0.07 & 94.91±0.52 & 4.13±0.44 & 94.95±1.89 & 2.67±0.06 & 94.94±0.43 \\
Eq$_{\text{C8}}$SocialVAE & 2.67±0.06 & 94.98±0.41 & 3.93±0.31 & 94.88±1.58 & 2.56±0.06 & 95.04±0.43 \\
Eq$_{\text{SO(2)}}$SocialVAE & 2.57±0.07 & 94.98±0.47 & 3.87±0.36 & 94.79±1.67 & 2.45±0.04 & 94.98±0.40 \\

\midrule
TUTR & 9.84±0.42 & 94.09±5.00 & 125.06±5.16 & 95.11±1.47 & 3.00±0.08 & 94.82±2.12 \\
Eq$_{\text{C4}}$TUTR & 6.67±1.38 & 95.02±3.88 & 121.41±12.17 & 94.94±1.34 & 2.02±0.38 & 94.86±2.92 \\
Eq$_{\text{C8}}$TUTR & 6.07±0.71 & 94.52±4.01 & 120.41±13.91 & 94.93±1.72 & 2.59±0.37 & 94.47±2.97 \\
Eq$_{\text{SO(2)}}$TUTR & 5.35±0.94 & 94.64±5.36 & 119.50±5.12 & 95.14±1.68 & 2.17±0.30 & 94.92±2.31 \\

\bottomrule
\end{tabular}
\end{adjustbox}

\vspace{1.2em}


\begin{adjustbox}{width=\textwidth}
\begin{tabular}{lcccccc}
\toprule
& \multicolumn{2}{c}{\textbf{UNIV}}
& \multicolumn{2}{c}{\textbf{ZARA$_1$}}
& \multicolumn{2}{c}{\textbf{ZARA$_2$}} \\
\cmidrule(lr){2-3}\cmidrule(lr){4-5}\cmidrule(lr){6-7}
\textbf{Method}
& $Q_{0.05}\downarrow$ & Cov$_{95\%}\uparrow$
& $Q_{0.05}\downarrow$ & Cov$_{95\%}\uparrow$
& $Q_{0.05}\downarrow$ & Cov$_{95\%}\uparrow$ \\
\midrule

SocialVAE & 3.81±0.05 & 95.03±0.27 & 3.54±0.06 & 94.98±0.42 & 3.61±0.08 & 94.86±0.51 \\
Eq$_{\text{C4}}$SocialVAE & 2.96±0.07 & 94.93±0.54 & 2.77±0.08 & 94.86±0.55 & 2.83±0.07 & 94.91±0.51 \\
Eq$_{\text{C8}}$SocialVAE & 2.84±0.05 & 94.90±0.43 & 2.67±0.06 & 94.93±0.48 & 2.71±0.07 & 94.98±0.55 \\
Eq$_{\text{SO(2)}}$SocialVAE & 2.78±0.07 & 94.95±0.53 & 2.57±0.06 & 94.88±0.52 & 2.61±0.05 & 94.90±0.43 \\

\midrule
TUTR & 3.39±0.09 & 95.01±0.70 & 2.67±0.41 & 95.03±1.74 & 2.78±0.13 & 94.74±0.93 \\
Eq$_{\text{C4}}$TUTR & 3.18±0.09 & 95.03±0.56 & 2.60±0.24 & 94.92±1.46 & 2.75±0.32 & 94.71±1.50 \\
Eq$_{\text{C8}}$TUTR & 3.08±0.04 & 94.90±0.38 & 2.73±0.16 & 94.83±1.15 & 2.92±0.20 & 94.73±1.23 \\
Eq$_{\text{SO(2)}}$TUTR & 2.95±0.05 & 94.89±0.46 & 2.56±0.28 & 94.80±1.63 & 2.59±0.21 & 94.66±1.46 \\

\bottomrule
\end{tabular}
\end{adjustbox}

\caption{
Calibration quantile ($Q_{\alpha=0.05}$) and empirical coverage (Cov$_{95\%}$) for 15-split conformal prediction under symmetry groups 
$\mathcal{G}\in\{\mathrm{SO(2), C4, C8}\}$ across ETH-UCY and SDD datasets. Our methods are dubbed $\text{Eq}_{\{\mathcal{G}\}}f_\theta$.
Lower $Q_{0.05}$ is better ($\downarrow$); higher coverage is better ($\uparrow$).
}
\label{table:0.05quantile}
\end{table*}
Table~\ref{table:0.05quantile} summarizes results on ETH--UCY and SDD. Across all datasets and base predictors, equivariantized conformal prediction consistently reduces the calibration quantile by approximately $28\%$ while maintaining coverage close to the nominal level. Equivariantized predictors achieve substantial reductions in $Q_{0.05}$, often exceeding 20--30\% relative to standard conformal prediction. This confirms the predicted contraction of high quantiles under increasing convex order and CVaR dominance, as indicated by Figure \ref{fig:eight}. Despite significantly tighter prediction sets, empirical coverage remains close to the target 95\% across all datasets and models, validating the finite-sample guarantees of eCP. Performance improves monotonically with the richness of the symmetry group: EqSO$(2)$ typically outperforms EqC$_8$, which in turn outperforms EqC$_4$. This aligns with the interpretation of equivariantization as orbit averaging, where larger groups induce stronger variance and tail contraction. Both SocialVAE and TUTR benefit from equivariantization, demonstrating that eCP applies equally well to stochastic and deterministic predictors.

Figure~\ref{fig:set_size} illustrates representative prediction sets produced by standard CP and eCP. Equivariantized prediction regions are visibly tighter, particularly at long horizons, while still enclosing the ground-truth trajectories. These qualitative results mirror the quantitative improvements observed in Table~\ref{table:0.05quantile}. Notably, the largest gains occur at high confidence levels, precisely where standard conformal prediction becomes overly conservative. This empirically supports the Chernoff bound tightening, CVaR contraction, and extreme-quantile improvements established in Section \ref{section:Symmetries-infused Conformal Prediction} and Appendices \ref{section:chernov}, \ref{section:hoeffding}.

\begin{table*}[h!]
\centering
\setlength{\tabcolsep}{5.5pt}
\renewcommand{\arraystretch}{1.15}

\begin{adjustbox}{width=\textwidth}
\begin{tabular}{lccccccc}
\toprule
\textbf{Method}
& \textbf{ETH}
& \textbf{HOTEL}
& \textbf{UNIV}
& \textbf{ZARA$_1$}
& \textbf{ZARA$_2$}
& \textbf{NBA Rebound}
& \textbf{NBA Score} \\
\midrule


TUTR
& 1.83/3.71
& 0.44/0.90
& 0.64/1.36
& 0.43/0.93
& 0.34/0.75
& --
& -- \\

Eq$_{\text{C4}}$TUTR
& 1.15/2.29
& 0.34/\textbf{0.67}
& 0.59/1.26
& 0.43/0.94
& \textbf{0.32/0.71}
& --
& -- \\

Eq$_{\text{C8}}$TUTR
& 1.12/2.26
& 0.39/0.77
& 0.55/1.19
& 0.49/1.09
& 0.35/0.79
& --
& -- \\

Eq$_{\text{SO(2)}}$TUTR
& \textbf{1.08/2.19}
& \textbf{0.37}/0.73
& \textbf{0.54/1.16}
& \textbf{0.49/1.08}
& 0.34/0.77
& --
& -- \\

\midrule

SocialVAE
& 0.58/1.28
& 0.54/1.22
& 0.64/1.41
& 0.58/1.30
& 0.60/1.35
& 1.86/4.23
& 2.08/4.91 \\

Eq$_{\text{C4}}$SocialVAE
& 0.46/1.02
& 0.43/0.96
& 0.51/1.13
& 0.47/1.04
& 0.48/1.06
& \textbf{1.33}/3.02
& 1.50/3.54 \\

Eq$_{\text{C8}}$SocialVAE
& 0.43/0.96
& 0.41/0.91
& 0.48/1.06
& 0.44/0.98
& 0.45/1.00
& \textbf{1.33/3.01}
& \textbf{1.50/3.53} \\

Eq$_{\text{SO(2)}}$SocialVAE
& \textbf{0.41/0.92}
& \textbf{0.39/0.87}
& \textbf{0.47/1.02}
& \textbf{0.42/0.94}
& \textbf{0.43/0.96}
& \textbf{1.33/3.01}
& \textbf{1.50/3.53} \\

\bottomrule
\end{tabular}
\end{adjustbox}

\caption{
Accuracy comparison (ADE/FDE) on ETH-UCY and NBA (rebound and scoring).
Lower is better.
}
\label{table:ade_fde_clean}
\end{table*}

\begin{figure*}[h!]
\centering

\begin{subfigure}[t]{0.5\linewidth}
  \centering
  \includegraphics[width=\linewidth]{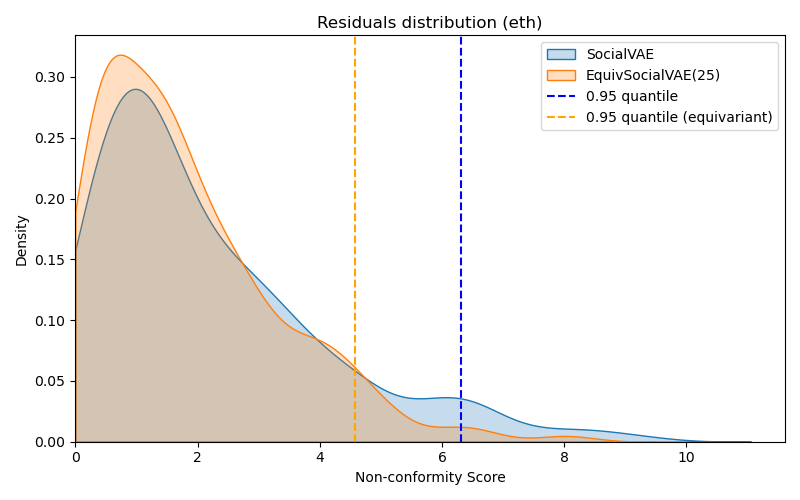}
\end{subfigure}\hfill
\begin{subfigure}[t]{0.5\linewidth}
  \centering
  \includegraphics[width=\linewidth]{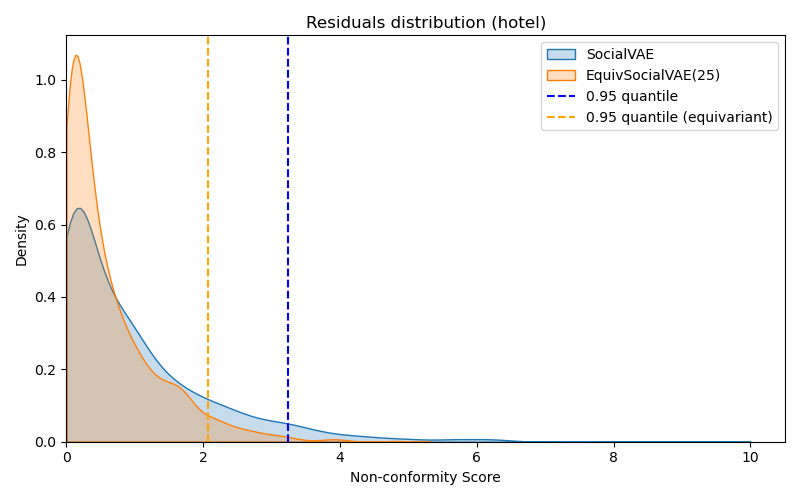}
\end{subfigure}\hfill
\begin{subfigure}[t]{0.5\linewidth}
  \centering
  \includegraphics[width=\linewidth]{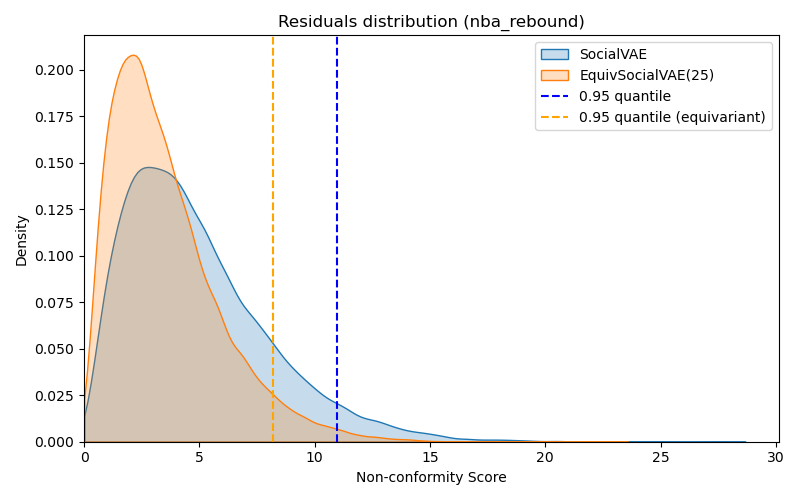}
\end{subfigure}\hfill
\begin{subfigure}[t]{0.5\linewidth}
  \centering
  \includegraphics[width=\linewidth]{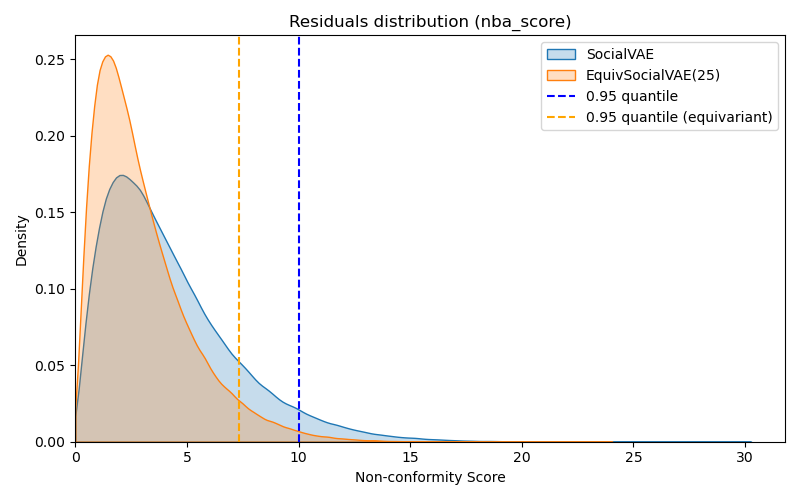}
\end{subfigure}
\begin{subfigure}[t]{0.5\textwidth}
  \centering
  \includegraphics[width=\linewidth]{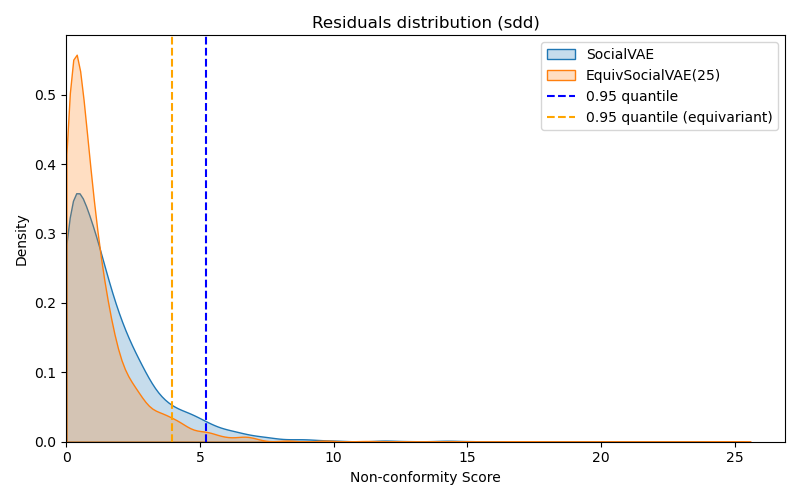}
\end{subfigure}\hfill
\begin{subfigure}[t]{0.5\textwidth}
  \centering
  \includegraphics[width=\linewidth]{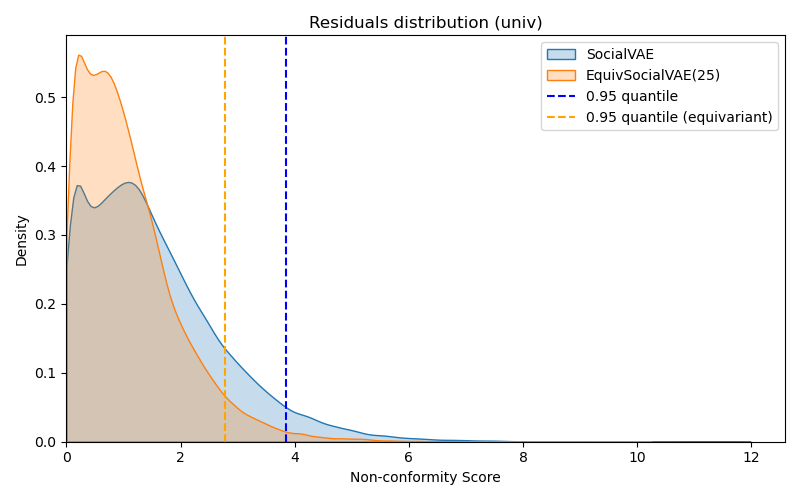}
\end{subfigure}\hfill
\begin{subfigure}[t]{0.5\textwidth}
  \centering
  \includegraphics[width=\linewidth]{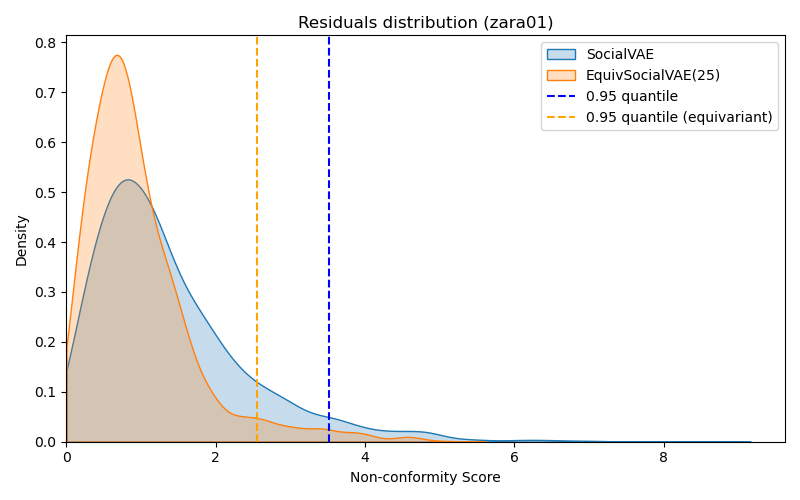}
\end{subfigure}\hfill
\begin{subfigure}[t]{0.5\textwidth}
  \centering
  \includegraphics[width=\linewidth]{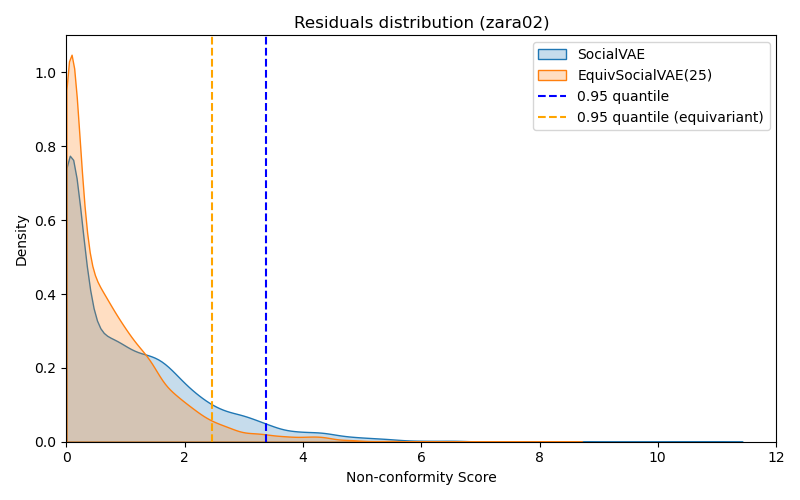}
\end{subfigure}

\caption{Non-conformity score distributions for SocialVAE and Eq${_{\text{SO2}}}$SocialVAE with $95^{\text{th}}$-quantile in dashed lines.}
\label{fig:eight}
\end{figure*}

\begin{table*}[h]
\centering
\setlength{\tabcolsep}{6pt}
\renewcommand{\arraystretch}{1.15}


\begin{adjustbox}{width=\textwidth}
\begin{tabular}{lcccccc}
\toprule
& \multicolumn{2}{c}{\textbf{ETH}}
& \multicolumn{2}{c}{\textbf{SDD}}
& \multicolumn{2}{c}{\textbf{HOTEL}} \\
\cmidrule(lr){2-3}\cmidrule(lr){4-5}\cmidrule(lr){6-7}
\textbf{Method}
& $Q_{0.01}\downarrow$ & Cov$_{99\%}\uparrow$
& $Q_{0.01}\downarrow$ & Cov$_{99\%}\uparrow$
& $Q_{0.01}\downarrow$ & Cov$_{99\%}\uparrow$ \\
\midrule

SocialVAE
& 5.14±0.19 & 98.98±0.21
& 9.22±2.60 & 98.89±0.60
& 4.93±0.26 & 98.98±0.29 \\

Eq$_{\text{C4}}$SocialVAE
& 4.26±0.16 & 98.97±0.24
& 7.08±1.59 & 98.93±0.64
& 4.10±0.21 & 98.97±0.30 \\

Eq$_{\text{C8}}$SocialVAE
& 4.06±0.13 & 98.97±0.20
& 6.57±0.93 & 98.88±0.74
& 3.98±0.16 & 98.99±0.23 \\

Eq$_{\text{SO(2)}}$SocialVAE
& 3.95±0.14 & 98.96±0.20
& 6.39±1.46 & 98.90±0.77
& 3.83±0.12 & 98.96±0.20 \\

\midrule
TUTR
& 11.51±3.29 & 97.88±2.12
& 208.32±61.59 & 98.81±1.05
& 4.72±0.15 & 98.98±0.21 \\

Eq$_{\text{C4}}$TUTR
& 8.05±0.99 & 97.66±2.34
& 204.23±88.41 & 98.89±0.97
& 2.88±0.92 & 98.68±1.10 \\

Eq$_{\text{C8}}$TUTR
& 7.40±1.20 & 98.08±1.92
& 199.65±47.48 & 98.81±0.86
& 3.55±0.28 & 98.88±1.01 \\

Eq$_{\text{SO(2)}}$TUTR
& 6.31±0.73 & 97.80±2.20
& 196.27±33.32 & 98.85±0.77
& 4.21±0.12 & 98.95±0.19 \\

\bottomrule
\end{tabular}
\end{adjustbox}

\vspace{1.2em}


\begin{adjustbox}{width=\textwidth}
\begin{tabular}{lcccccc}
\toprule
& \multicolumn{2}{c}{\textbf{UNIV}}
& \multicolumn{2}{c}{\textbf{ZARA$_1$}}
& \multicolumn{2}{c}{\textbf{ZARA$_2$}} \\
\cmidrule(lr){2-3}\cmidrule(lr){4-5}\cmidrule(lr){6-7}
\textbf{Method}
& $Q_{0.01}\downarrow$ & Cov$_{99\%}\uparrow$
& $Q_{0.01}\downarrow$ & Cov$_{99\%}\uparrow$
& $Q_{0.01}\downarrow$ & Cov$_{99\%}\uparrow$ \\
\midrule

SocialVAE
& 5.38±0.21 & 98.97±0.25
& 5.08±0.20 & 98.99±0.22
& 5.29±0.37 & 98.96±0.43 \\

Eq$_{\text{C4}}$SocialVAE
& 4.39±0.22 & 98.98±0.30
& 4.22±0.20 & 98.97±0.27
& 4.24±0.14 & 98.98±0.20 \\

Eq$_{\text{C8}}$SocialVAE
& 4.22±0.18 & 98.99±0.24
& 4.07±0.13 & 98.98±0.20
& 4.11±0.13 & 98.98±0.22 \\

Eq$_{\text{SO(2)}}$SocialVAE
& 4.17±0.17 & 98.98±0.28
& 3.93±0.21 & 98.97±0.28
& 3.98±0.13 & 98.99±0.24 \\

\midrule
TUTR
& 4.72±0.15 & 98.98±0.21
& 4.21±0.66 & 98.95±0.71
& 4.48±0.31 & 98.92±0.43 \\

Eq$_{\text{C4}}$TUTR
& 4.45±0.12 & 98.99±0.20
& 4.23±0.44 & 99.03±0.74
& 4.46±0.41 & 98.89±0.64 \\

Eq$_{\text{C8}}$TUTR
& 4.38±0.27 & 98.98±0.35
& 4.12±0.34 & 99.01±0.59
& 4.14±0.24 & 98.89±0.57 \\

Eq$_{\text{SO(2)}}$TUTR
& 4.21±0.12 & 98.95±0.19
& 4.04±0.25 & 99.02±0.53
& 3.92±0.15 & 98.94±0.29 \\

\bottomrule
\end{tabular}
\end{adjustbox}

\caption{
Calibration quantile ($Q_{\alpha=0.01}$) and empirical coverage (Cov$_{99\%}$) for 15-split conformal prediction ($K=20$ samples) under symmetry groups 
$\mathcal{G}\in\{\mathrm{SO(2), C4, C8}\}$.
Lower is better for $Q_{0.01}$; higher is better for coverage.
}
\label{table:quantile_001_two_row}
\end{table*}

\begin{figure*}[h!]
\centering
\begin{subfigure}[t]{0.7\linewidth}
  \centering
  \includegraphics[width=\linewidth]{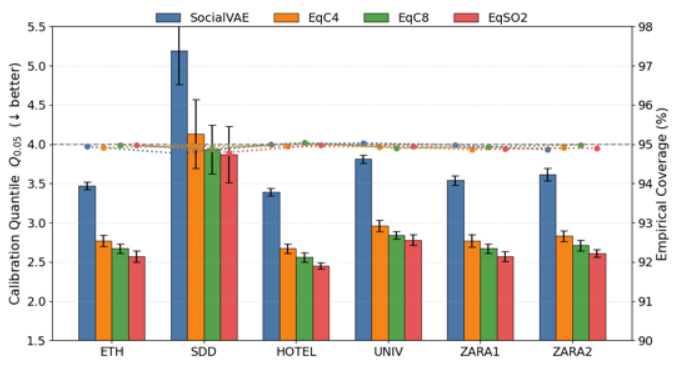}
\end{subfigure}\hfill
\caption{Calibration quantile ($Q_{\alpha=0.05}$) and empirical coverage (Cov$_{95\%}$) for 15-split conformal prediction under symmetry groups 
$\mathcal{G}\in\{\mathrm{SO(2), C4, C8}\}$ across ETH-UCY and SDD datasets. eCP methods consistently lead to smaller calibration quantiles, while preserving the prescribed coverage rate.}
\end{figure*}

These results confirm that explicitly injecting symmetry into post-hoc uncertainty quantification yields tangible improvements in efficiency without sacrificing validity. Importantly, eCP requires neither architectural equivariance nor retraining: approximate or emergent symmetry in pretrained models is sufficient to achieve meaningful uncertainty reduction.
Overall, the experiments demonstrate that equivariantized conformal prediction provides a principled and practical mechanism for sharpening uncertainty estimates in symmetry-rich, long-horizon forecasting tasks.

\section{Future work \& Limitations}
While eCP consistently improves the efficiency of conformal prediction sets, several limitations merit discussion. First, the method relies on access to a known or approximately valid symmetry group. If the assumed group poorly reflects the true invariances of the data, equivariantization may yield limited gains. Second, richer symmetry groups incur additional computational cost due to orbit averaging, particularly for continuous groups such as $\mathrm{SO}(2)$, which require Monte Carlo approximation. Although this cost is modest relative to model retraining, it may become non-negligible for some real-time systems with strict latency constraints. Finally, our experiments focus on prediction-set tightness and coverage rather than closed-loop planning performance. While tighter uncertainty sets are strongly correlated with improved safety and efficiency in downstream planners, formally quantifying this effect in an integrated planning pipeline remains an important direction for future work.




\section*{APPENDIX I - Chernoff bound improvement via MGF ordering}
\label{section:chernov}
\begin{lemma}[MGF ordering]
\label{lem:mgf}
Consider the cumulant generating functions $\psi_f(\lambda) = \log \mathbb{E}[e^{\lambda S_f}]$ and $\psi_{f^G}(\lambda) = \log \mathbb{E}[e^{\lambda S_{f^G}}]$. For $G$-invariant distribution $P$ and, if Assumptions \ref{as:G-invariant_score},\ref{as:convex_score} stand, then:
\[
\psi_{f^G}(\lambda) \le \psi_f(\lambda) \quad \forall\lambda \ge 0.
\]
\end{lemma}

\begin{theorem}[non-asymptotic Chernoff bound improvement]
\label{thm:chernoff}
For any threshold $t>0$,
\begin{align*}
    &\mathbb{B}_{\mathrm{Chernov}}(S_{f^G} \ge t)
\;:=\;
\inf_{\lambda \ge 0} e^{-\lambda t + \psi_{f^G}(\lambda)}\\
&\qquad\qquad\;\le\;
\inf_{\lambda \ge 0} e^{-\lambda t + \psi_f(\lambda)}
\;:=\;
\mathbb{B}_{\mathrm{Chernov}}(S_f \ge t)
\end{align*}
\end{theorem}
Thus, every exponential tail bound (Chernoff, Bernstein, Bennett) for $S_f$ is uniformly tightened under symmetrization of the predictor. 

\noindent\textit{Rate-function dominance and asymptotic consequences}:
Assume the MGF $\mathbb{E}[e^{\lambda S_f}]$ is finite in a neighborhood of $0$. Define the Legendre transform (the \emph{rate function} from theory of Large Deviations) $I_{S_f}(t) = \sup_{\lambda \ge 0} \{ \lambda t - \psi_f(\lambda)\}$ and $I_{S_{f^G}}(t) = \sup_{\lambda \ge 0} \{ \lambda t - \psi_{f^G}(\lambda)\}$. Another consequence of the MGF ordering is that the Cram\'er rate function of the equivariantized model dominates the pre-trained non-symmetric one.

\begin{cor}[Rate-function dominance]
\label{lem:ratefunc}
If Lemma \ref{lem:mgf} stands, i.e. $\psi_{f^G}(\lambda)\le\psi_f(\lambda),\forall\lambda$, then $I_{S_{f^G}}(t)\ge I_{S_f}(t),\forall t$.
\end{cor}

The ordering $\psi_{f^G}\!\le\!\psi_f$ means equivariantization suppresses large-deviation mass by averaging over symmetry-related orbits, yielding a steeper rate function and therefore faster tail decay. 
Geometrically, symmetry alignment redistributes uncertainty within each orbit, producing sharper prediction regions while preserving finite-sample coverage.

\section*{APPENDIX II - Tightening Hoeffding's Bound}
\label{section:hoeffding}

The preceding sections established that equivariantization reduces both the variance and the moment generating function of the nonconformity scores.
We now show that these properties also lead to an improvement in Hoeffding-type concentration, which controls the deviation of empirical averages from their expectations and underpins many finite-sample coverage guarantees in conformal prediction.

Consider $f_\theta$ model trained on $D_{\text{train}}$ and $D_{\text{cal}} = {(X_i, Y_i)}_{i=1:n_c}$ calibration set, both drawn i.i.d. from a G-invariant distribution. For $G$-invariant and convex $s:\mathcal{Y}\times\mathcal{Y}\rightarrow \mathbb{R}_{\geq0}$, consider $S_{f,i}=s(f(X_i),Y_i)$ and $S_{f^G,i}=s(f^G(X_i),Y_i)$ non-conformity scores on $D_{\text{cal}}$ for the pre-trained model and its equivariantized form. Assume $S_f\in [0,b]\subset \mathbb{R}_{\geq0}$ almost surely and define the empirical statistics $$\bar S_{n_c}=\dfrac{1}{n_c}\sum_{i=1}^{n_c}S_{f,i} \quad,\quad\bar S^G_{n_c}=\dfrac{1}{n_c}\sum_{i=1}^{n_c}S_{f^G,i} \quad $$
 Hoeffding's inequality yields
\begin{align*}
P\!\left(
  \left| \bar S_{n_c}- \mathbb{E}[S_f]
  \right| > \varepsilon
\right)
\le
2\exp\!\left(
  -\tfrac{2n_c\varepsilon^2}{b^2}
\right)=:\mathbb{B}_{\text{Hoeffding}}(\bar S_{n_c})
\label{eq:hoeffding-equ}
\end{align*}

\begin{lemma}\label{lemma:hoeffding}
 Equivariantized predictors \ref{eq:equivariantized_predictor} yield tighter Hoeffding-type bounds, i.e. $ \mathbb{B}_{\text{Hoeffding}}(\bar S_{n_c}) \geq \mathbb{B}_{\text{Hoeffding}}(\bar S^G_{n_c})$
\end{lemma}
Equivariantization therefore provides a potentially tighter exponential tail bound whenever the score varies nontrivially across group orbits, as averaging (a projection), so it cannot expand the score range; it preserves or contracts it. A sharper bound follows from Bernstein’s inequality, which explicitly depends on the variance:
\begin{equation}
P\!\left(
  \bar S^G_{n_c} - \mathbb{E}[S_{f^G}] > \varepsilon
\right)
\le
\exp\!\left(
  -\frac{
    n_c\varepsilon^2
  }{
    2\,\mathrm{Var}(S_{f^G})
    +\tfrac{2}{3}b_G\varepsilon
  }
\right)
\label{eq:bernstein}
\end{equation}
Since $S_{f^G} \preceq_{\mathrm{icx}}  S_f $ implies $\mathrm{Var}(S_{f^G}) \le \mathrm{Var}(S)$ and $b_G \le b$, then $\mathbb{B}_{\text{Bernstein}}(S_{f^G}) \leq \mathbb{B}_{\text{Bernstein}}(S_{f})$. This guarantees at least as strong concentration of calibration statistics around their expectations, which in turn yields tighter empirical quantile estimates and smaller conformal prediction sets. Furthermore, from Equation \ref{eq:equivariantized_mean_score_decrease} $\mathbb{E}[S_{f^G}]\leq \mathbb{E}[S_f]$, meaning that the empirical equivariantized non-conformity scores concentrate faster around a smaller statistical mean. Variance reduction leads to tightening of Chebysev-Cantelli's concentration inequality
\begin{equation}
P\!\left(
  S_{f^G} - \mathbb{E}[S_{f^G}] > \varepsilon
\right)
\le \dfrac{\mathrm{Var}(S_{f^G})}{\mathrm{Var}(S_{f^G}) + \varepsilon^2}
\label{eq:bernstein}
\end{equation}

\section*{APPENDIX III - Proofs}
\begin{proof}[Proof of Lemma \ref{lem:PiG_invariance_projection}]
By Definition~\ref{def:symmetrization_operator_invariance} and left-invariance of the Haar probability measure $\mu_G$,
\begin{align*}
&\Pi_G[s;f]\big(\phi_h(x),\psi_h(y)\big)\\
&=\int_G s\!\Big( f(\phi_{g^{-1}}\phi_h(x)) ,\,\psi_{g^{-1}}\psi_h(y)\Big)\,d\mu_G(g)\\
&\stackrel{g'=h^{-1}g}{=}\int_{h\cdot G} s\!\Big(f(\phi_{g'^{-1}}(x)),\,\psi_{g'^{-1}}(y)\Big)\,d\mu_G(hg') \\
&=\int_G s\!\Big(f(\phi_{g'^{-1}}(x)),\,\psi_{g'^{-1}}(y)\Big)\,d\mu_G(g')=\Pi_G[s;f](x,y),
\end{align*}
where we used $d\mu_G(g)=d\mu_G(g')$ and the group homomorphism property of the actions.
If $f\in F_G$, then $f\circ\phi_{g^{-1}}=\psi_{g^{-1}}\circ f$, hence by $G$-invariance of $s$ in Assumption~\ref{as:G-invariant_score}, $s\big(f(\phi_{g^{-1}}(x)),\psi_{g^{-1}}(y)\big)=s\big(\psi_{g^{-1}}f(x),\psi_{g^{-1}}(y)\big)=s\big(f(x),y\big)$, so
\[
\Pi_G[s;f](x,y)=\int_G s\big(f(x),y\big)\,d\mu_G(g)=s\big(f(x),y\big).
\]
\end{proof}

\begin{proof}[Proof of Lemma \ref{lem:expectation_orbit_conditional}]
    Let $(X,Y)\in\mathcal{O}^{(\phi,\psi)}_{(x,y)}$. For some element $g\in G$ it stands that
    \begin{align*}
        &\mathbb{E}_{(\mathcal{X}\times \mathcal{Y})\sim P}[s(f(X),Y) \,|\, (X,Y)\in\mathcal{O}^{(\phi,\psi)}_{(x,y)} ] \\&=\int_{\mathcal{O}^{(\phi,\psi)}_{(x,y)}} s\big(f(X),Y \big) d\mu(X,Y) \\
        &=
        \int_{\mathcal{X,Y}} s\Big(f\big( \phi_{g^{-1}}(X)\big),\psi_{g^{-1}}(Y)  \Big) \cdot\\
        &\qquad\qquad\qquad \cdot\mathds{1}\{(\phi_{g^{-1}}(X),\psi_{g^{-1}}(Y) )\in \mathcal{O}^{(\phi,\psi)}_{(x,y)}\} d\mu(X,Y) \\
        &=  \int_{\mathcal{X,Y}} s\Big(f\big( \phi_{g^{-1}}(X)\big),\psi_{g^{-1}}(Y)  \Big) \cdot\\
        &\qquad\qquad\qquad \cdot\mathds{1}\{(X,Y )\in \mathcal{O}^{(\phi,\psi)}_{(x,y)}\} d\mu(X,Y)
    \end{align*}
    By construction of the orbit, $\exists g^*\in G\;\text{s.t.}\; x=\phi_{g^*}X$. Taking expectations over the group yields:
    \begin{align*}
        &\int_{\mathcal{O}^{(\phi,\psi)}_{(x,y)}} s\big(f(X),Y \big) d\mu(X,Y)=  \\
        &= \int_G \int_{\mathcal{X,Y}} s\Big(f\big( \phi_{g^{-1}}(X)\big),\psi_{g^{-1}}(Y)  \Big) \cdot\\
        &\qquad\qquad\qquad \cdot\mathds{1}\{(X,Y )\in \mathcal{O}^{(\phi,\psi)}_{(x,y)}\} d\mu(X,Y) d\mu_G(g)=\\
        &= \int_G \int_{\mathcal{X,Y}} s\Big(f\big( \phi_{g^{-1}g^{*-1}}(x)\big),\psi_{g^{-1}g^{*-1}}(y)  \Big)\cdot\\
        &\qquad\qquad\qquad \cdot \mathds{1}\{(X,Y )\in \mathcal{O}^{(\phi,\psi)}_{(x,y)}\} d\mu(X,Y) d\mu_G(g)=\\
        &\stackrel{Fubini}{=} \int_{\mathcal{X,Y}} \int_G s\Big(f\big( \phi_{g^{-1}g^{*-1}}(x)\big),\psi_{g^{-1}g^{*-1}}(y)  \Big) d\mu_G(g) \cdot\\
        &\qquad\qquad\qquad \cdot\mathds{1}\{(X,Y )\in  \mathcal{O}^{(\phi,\psi)}_{(x,y)}\} d\mu(X,Y) =\\
        &\stackrel{\text{Lemma \ref{lem:PiG_invariance_projection}}}{=} \int_{\mathcal{X,Y}} \int_G s\Big(f\big( \phi_{g^{-1}}(x)\big),\psi_{g^{-1}}(y)  \Big) d\mu_G(g) \cdot\\
        &\qquad\qquad\qquad \cdot\mathds{1}\{(X,Y )\in  \mathcal{O}^{(\phi,\psi)}_{(x,y)}\} d\mu(X,Y)=\\
        &= \int_{\mathcal{O}^{(\phi,\psi)}_{(x,y)}} \Pi_G[s;f](x,y)  d\mu(X,Y) = \Pi_G[s;f](x,y) 
    \end{align*}
\end{proof}

\begin{proof}[Proof of Theorem \ref{thm:split-g-equivariant-validity}]
By Assumption~\ref{as:G-invariant_score}, $\Pi[s;f]$ is the same measurable map applied to every index and is $G$-invariant:
\[
\Pi[s;f](g\!\cdot\!x,\,g\!\cdot\!y)=\Pi[s;f](x,y)\qquad\forall\,g\in G.
\]
Because the calibration points and test point are drawn from a $G$-invariant law and are permutation-exchangeable across indices (as stated in the theorem), the $(m{+}1)$-tuple
\[
\big(\,(X_1,Y_1),\ldots,(X_m,Y_m),\,(X_{m+1},Y_{m+1})\,\big)
\]
is exchangeable under index permutations, conditional on $\mathcal D_{\text{train}}$.
Applying the same coordinate-wise function $(x,y)\mapsto \Pi[s;f](x,y)$ preserves exchangeability.
Therefore the \emph{score vector}
\begin{align*}
    &\big(\,\widetilde S_1,\ldots,\widetilde S_m,\,\widetilde S_{m+1}(Y_{m+1})\,\big)
\\ \text{with}\quad
\widetilde S_i:=&\Pi[s;f](X_i,Y_i),\;\;\widetilde S_{m+1}(y):=\Pi[s;f](X_{m+1},y)
\end{align*}
is exchangeable conditional on $\mathcal D_{\text{train}}$.

Let $\widetilde S_{(1)}\le\cdots\le \widetilde S_{(m)}$ denote the ordered calibration scores and set
\(
k=\lceil (m+1)(1-\alpha)\rceil,\; q=\widetilde S_{(k)}.
\)
Define the rank
\[
R \;:=\; \big|\{i\in\{1,\ldots,m{+}1\}:\, \widetilde S_i \ge \widetilde S_{m+1}(Y_{m+1})\}\big|.
\]
By exchangeability of the score vector, conditional on $\mathcal D_{\text{train}}$ the rank $R$ is uniform on $\{1,\ldots,m{+}1\}$ after standard randomized tie-breaking; without randomization, it is stochastically no smaller than the uniform law, yielding a conservative inequality. By construction, $\widetilde S_{m+1}(Y_{m+1})\le q$ if and only if $R\le k$.
Hence,
\(
P\!\big(\widetilde S_{m+1}(Y_{m+1})\le q \,\big|\, \mathcal D_{\text{train}}\big)
=\Pr(R\le k \mid \mathcal D_{\text{train}})
\ge \frac{k}{m+1}
\ge 1-\alpha.
\)
But $\{\widetilde S_{m+1}(Y_{m+1})\le q\}$ is exactly the event $\{Y_{m+1}\in C_{1-\alpha}^{(G)}(X_{m+1})\}$, which proves
\begin{align*}
    P\!\big(Y_{m+1}\in C_{1-\alpha}^{(G)}(X_{m+1}) \,\big|\, \mathcal D_{\text{train}}\big)\;\ge\; 1-\alpha
\end{align*}
If ties occur with probability zero (e.g., when $\Pi[s;f](X,Y)$ has a continuous distribution), the non-randomized rule attains coverage $k/(m{+}1)$, which differs from $1-\alpha$ by at most $1/(m{+}1)$; with randomized tie-breaking it equals $1-\alpha$ exactly.
Moreover, $G$-invariance of $\Pi[s;f]$ implies the natural equivariance of the predictor:
\(
\widetilde S_{m+1}(y)\le q \iff \widetilde S_{m+1}(\psi_g(y))\le q\,,\,\forall g\in G,
\)
so $C_{1-\alpha}^{(G)}(\phi_g(x))=\psi_g \circ C_{1-\alpha}^{(G)}(x)\,,\,\forall g\in G$.
\end{proof}

\begin{proof}[Proof of Corollary \ref{cor:variance_decomp_symm }]
    Let $P$ be a $G$-invariant probability measure on $\mathcal{X}\times\mathcal{Y}$ and let $Z:=s(f(X),Y)\in L^2(P)$. Lemma \ref{lem:expectation_orbit_conditional} identified $\Pi_G$ with a conditional expectation onto $G$-orbits, i.e.:
    \[
    \Pi_G[s;f](X,Y)\;=\;\mathbb{E}\!\big[Z\,\big|\,\sigma(G)\big],
    \]
    where $\mu_G$ is the Haar probability measure on the compact group $G$ and $\sigma(G)$ is the $\sigma$-algebra of $G$-invariant events. The law of total variance yields:
    \begin{align*}
    \mathrm{Var}(Z)
    &=\mathrm{Var}\!\big(\mathbb{E}[Z\mid \sigma(G)]\big)
    +\mathbb{E}\!\big[\mathrm{Var}(Z\mid \sigma(G))\big]\\
    &=\mathrm{Var}\!\big(\Pi_G[s;f]\big)+\mathbb{E}\!\big[\mathrm{Var}(Z\mid \sigma(G))\big]
    \end{align*}
    and hence $\mathrm{Var}(Z)-\mathrm{Var}\!\big(\Pi_G[s;f]\big)\;=\;\mathbb{E}\!\big[\mathrm{Var}(Z\mid\sigma(G))\big]\;\ge 0$, with equality iff $Z\sim \hat{P}$ with $\hat{P}$ $G$-invariant almost surely.
\end{proof}

\begin{proof}[Proof of Theorem \ref{thm:icx}]
For $(X,Y)\sim P$, from pointwise Jensen's inequality along $G$-orbits \ref{eq:jENSEN_INEQUALITY}, we have 
\begin{equation}\label{eq:pointwise-dom}
s\big(f_\theta^G(X),Y\big)\;\le\;\Pi_G[s;f_\theta](X,Y)\qquad\text{a.s.}
\end{equation}
By Lemma~\ref{lem:expectation_orbit_conditional}, one may identify $\Pi_G[s;f_\theta]$ as a conditional expectation:
\[
\Pi_G[s;f_\theta](X,Y)
\;=\;\mathbb{E}\!\big[s(f_\theta(X),Y)\,\big|\,\sigma(G)\big]
\;=\;\mathbb{E}\!\big[S_f\,\big|\,\sigma(G)\big],
\]
where $\sigma(G)$ is the $\sigma$-algebra of $G$-invariant events (equivalently, the $\sigma$-algebra generated by the orbits). From \eqref{eq:pointwise-dom} and the monotonicity of $\varphi$:
\[
\varphi\!\big(S_{f^G}\big)\;\le\;\varphi\!\Big(\mathbb{E}\!\big[S_f\mid\sigma(G)\big]\Big)\qquad\text{a.s.}
\]
By conditional Jensen (convexity of $\varphi$):
\[
\varphi\!\Big(\mathbb{E}\!\big[S_f\mid\sigma(G)\big]\Big)\;\le\;\mathbb{E}\!\big[\varphi(S_f)\mid\sigma(G)\big]\qquad\text{a.s.}
\]
Taking expectations yields
\[
\mathbb{E}\big[\varphi(S_{f^G})\big]\;\le\;\mathbb{E}\big[\varphi(S_f)\big].
\]
Since this holds for every increasing convex $\varphi$ with finite expectation, we conclude $S_{f^G}\preceq_{\mathrm{icx}} S_f$.
\end{proof}

\begin{proof}[Proof of Theorem \ref{theorem:cvar_ordering}]
By Theorem~\ref{thm:icx}, we have
$S_{f^G}\preceq_{\mathrm{icx}} S_f$. 
By the stop-loss characterization of $\preceq_{\mathrm{icx}}$,
\begin{equation}\label{eq:stoploss}
\mathbb{E}\big[(S_{f^G}-t)_+\big]\;\le\;\mathbb{E}\big[(S_f-t)_+\big]\qquad\forall\,t\in\mathbb{R}.
\end{equation}
Recall the Rockafellar--Uryasev representation of CVaR for any integrable random variable $Z$ and $\alpha\in[0,1)$:
\[
\mathrm{CVaR}_\alpha(Z)
=\inf_{t\in\mathbb{R}}\Big\{\, t \;+\; \frac{1}{1-\alpha}\,\mathbb{E}[(Z-t)_+] \,\Big\}.
\]
Applying this to $Z=S_{f^G}$ and $Z=S_f$, and using \eqref{eq:stoploss}, we obtain $\forall t\in\mathbb{R}$,
\[
t+\frac{1}{1-\alpha}\mathbb{E}[(S_{f^G}-t)_+]
\;\le\; t+\frac{1}{1-\alpha}\mathbb{E}[(S_f-t)_+].
\]
Taking the infimum over $t$ on both sides yields
\[
\mathrm{CVaR}_\alpha(S_{f^G})
\;\le\; \mathrm{CVaR}_\alpha(S_f).
\]
Furthermore, from the definition of $\textrm{CVaR}_\alpha$, it holds that:
\begin{align*}
    &\mathrm{CVaR}_\alpha(S_f)- \mathrm{CVaR}_\alpha(S_{f^G}) = \frac{1}{1-\alpha} \int_\alpha^1 F_{S_{f}}^{-1}(u) - F_{S_{f^G}}^{-1}(u)\,du \\
    &=\frac{1}{1-\alpha} \Bigg[\int_0^1 F_{S_{f}}^{-1}(u) - F_{S_{f^G}}^{-1}(u)du -\underbrace{\int_0^\alpha F_{S_{f}}^{-1}(u) - F_{S_{f^G}}^{-1}(u)du}_{\leq0 \text{  from Lemma \ref{lemma:2.2}}} \Bigg] \\
    &\leq \frac{1}{1-\alpha} \Big[\int_0^1 F_{S_{f}}^{-1}(u)\,du - \int_0^1 F_{S_{f^G}}^{-1}(u)\,du \Big] = \dfrac{\mathbb{E}[S_f - S_{f^G}]}{1-\alpha}
\end{align*}
\end{proof}

\begin{proof}[Proof of Lemma \ref{lemma:2.2}]
Recall that for an integrable random variable $X$, the \emph{stop-loss transform} is defined by
\[
\mathrm{SL}_X(t) := \mathbb{E}(X - t)_+ = \int_t^{\infty} \big(1 - F(x)\big)\,dx, 
\qquad t \in \mathbb{R}.
\]
For each $p \in (0,1)$, it holds that
    \begin{equation}
    \int_p^1 F^{-1}(u)\,du 
    = \inf_{t \in \mathbb{R}} 
      \big\{ \mathrm{SL}_X(t) + t(1-p) \big\},
    \label{eq:A}
    \end{equation}
and the infimum is attained at any $t$ such that $F(t) = p$. Conversely, $\forall t \in \mathbb{R}$,
    \begin{equation}
    \mathrm{SL}_X(t) 
    = \sup_{p \in [0,1]} 
      \bigg\{ \int_p^1 F^{-1}(u)\,du - t(1-p) \bigg\}.
    \label{eq:B}
    \end{equation}
These formulas show that the functions 
$t \mapsto \mathrm{SL}_X(t)$ and $p \mapsto \int_p^1 F^{-1}(u)\,du$
are Legendre–Fenchel conjugates up to a linear change of variables.

\vspace{0.5em}
\noindent\textbf{($\Rightarrow$)} 
Assume that $X \le_{icx} Y$.
By definition of the increasing convex order, this is equivalent to 
\[
\mathrm{SL}_X(t) \le \mathrm{SL}_Y(t), \qquad \forall\, t \in \mathbb{R}.
\]
Then, for any $p \in (0,1)$, applying \eqref{eq:A} to $X$ and $Y$ yields
\begin{align*}
\int_p^1 F^{-1}(u)\,du
&= \inf_t \big\{\mathrm{SL}_X(t) + t(1-p)\big\} \\
&\le \inf_t \big\{\mathrm{SL}_Y(t) + t(1-p)\big\}\\
&= \int_p^1 G^{-1}(u)\,du.
\end{align*}
Hence the stated inequality holds for all $p$.

\vspace{0.5em}
\noindent\textbf{($\Leftarrow$)} 
Conversely, assume that 
\[
\int_p^1 F^{-1}(u)\,du \le \int_p^1 G^{-1}(u)\,du
\quad \forall\,p \in (0,1).
\]
By \eqref{eq:B}, for each $t \in \mathbb{R}$,
\begin{align*}
\mathrm{SL}_X(t)
&= \sup_{p \in [0,1]} 
  \Big\{\int_p^1 F^{-1}(u)\,du - t(1-p)\Big\}\\
&\le \sup_{p \in [0,1]} 
  \Big\{\int_p^1 G^{-1}(u)\,du - t(1-p)\Big\}
= \mathrm{SL}_Y(t).
\end{align*}
Thus, $\mathrm{SL}_X(t) \le \mathrm{SL}_Y(t)$ for all $t$, 
and therefore $X \le_{icx} Y$. Combining both directions, we have established that
\[
X \le_{icx} Y
\; \Longleftrightarrow \;
\int_p^1 F^{-1}(u)\,du \le \int_p^1 G^{-1}(u)\,du
\quad \forall\, p \in (0,1)
\]
\end{proof}

\begin{proof}[Proof of Corollary \ref{cor:expetced_quantile_contraction}]
From Theorem \ref{thm:icx}, $S_{f^G} \preceq_{\mathrm{icx}} S_f$, meaning that $\mathbb{E}[\varphi(S_{f^G})]
\;\le\;
\mathbb{E}[\varphi(S_f)]$ for all increasing convex functions $\varphi$.
Choosing $\varphi(t)=t$, which is increasing and convex, yields
\begin{equation}\label{eq:means}
\mathbb{E}[S_{f^G}] \;\le\; \mathbb{E}[S_f].
\end{equation}

Next, recall the standard identity relating a random variable to its
quantile function.  
If $X$ has quantile function $F_X^{-1}$ and $U\sim\mathrm{Unif}(0,1)$,
then $F_X^{-1}(U)\stackrel{d}{=}X$, and therefore
\begin{equation}\label{eq:quantile-mean}
\mathbb{E}[X]
= \mathbb{E}[F_X^{-1}(U)]
= \int_0^1 F_X^{-1}(u)\,du.
\end{equation}

Applying \eqref{eq:quantile-mean} to $S_{f^G}$ and $S_f$ and using
\eqref{eq:means} gives
\[
\mathbb{E}\!\left[F^{-1}_{S_{f^G}}(U)\right]
= \mathbb{E}[S_{f^G}]
\;\le\;
\mathbb{E}[S_f]
= \mathbb{E}\!\left[F^{-1}_{S_f}(U)\right].
\]
Equivalently,
\(
\int_0^1 F^{-1}_{S_{f^G}}(u)\,du
\;\le\;
\int_0^1 F^{-1}_{S_f}(u)\,du.
\)

Now, let $X \leq_{icx} Y$ and $\phi$ be an increasing convex function. 
Then, it is well known that $\phi(X) \leq_{icx} \phi(Y)$. 
By Lemma~\ref{lemma:2.2}, this is equivalent to saying
\begin{equation}
\int_p^1 F_{\phi}^{-1}(t)\, dt 
   \leq 
   \int_p^1 G_{\phi}^{-1}(t)\, dt
\label{eq:2.2}
\end{equation}
$\forall p \in (0,1)$ and for all increasing and convex $\phi$, where 
$F_{\phi}^{-1}(t) = \phi\!\left(F^{-1}(t)\right)$ 
and 
$G_{\phi}^{-1}(t) = \phi\!\left(G^{-1}(t)\right)$ 
are the quantile functions of $\phi(X)$ and $\phi(Y)$, respectively. Evidently, \eqref{eq:2.2} is equivalent to
\begin{equation}
\int_{G(x)}^1 F_{\phi}^{-1}(t)\, dt 
   \leq 
   \int_{G(x)}^1 G_{\phi}^{-1}(t)\, dt
\label{eq:2.3}
\end{equation}
$\forall x \in \mathbb{R} $ and for all increasing convex $\phi$.
Since
\[
\frac{\int_p^1 F^{-1}(t)\, dt}{1-p}
   = \mathbb{E}\!\left[ X \mid X > F^{-1}(p) \right],
\]
\eqref{eq:2.3} is equivalent to
\begin{align}
&\mathbb{E}\!\left[ \phi(X) \mid \phi(X) > F_{\phi}^{-1}\!\big(G(x)\big) \right]\nonumber\\
&\qquad\qquad\leq \mathbb{E}\!\left[ \phi(Y) \mid \phi(Y) > G_{\phi}^{-1}\!\big(G(x)\big) \right]\label{eq:2.4}
\end{align}
Since $X \stackrel{d}{=} F^{-1} \circ G(Y)$, \eqref{eq:2.4} implies that $$\mathbb{E}\left[ \phi\big(F^{-1} \circ G(Y)\big) \,\middle|\, Y > x \right]
   \leq 
   \mathbb{E}\left[ \phi(Y) \,\middle|\, Y > x \right]$$
\end{proof}

\begin{proof}[Proof of Lemma \ref{lem:mgf}]
By Jensen's inequality applied to the exponential function and the convexity of $s$, we have
\begin{align*}
&\mathbb{E}\bigl[e^{\lambda s(f_\theta^G(X),Y)}\bigr]
= \mathbb{E}\Bigl[e^{\lambda s\bigg( \mathbb{E}_G\Big[\psi_g \circ f_\theta \circ \phi_{g^{-1}}(X)\Big], Y \bigg)}\Bigr]\\
&\qquad\le \mathbb{E}\Bigl[\mathbb{E}_G[e^{\lambda s(\psi_g \circ f_\theta \circ \phi_{g^{-1}}(X),Y)}]\Bigr]
= \mathbb{E}\bigl[e^{\lambda s(f_\theta(X),Y)}\bigr],
\end{align*}
using the $G$-invariance of $(X,Y)\sim P$. Taking the logarithmic functions yields the result.
\end{proof}

\begin{proof}[Proof of Lemma \ref{lemma:hoeffding}]
    From Jensen's inequality $$s(f_\theta^G(X_i),Y_i)\leq \mathbb{E}_g\Big[s\Big(f_\theta\big(\phi_{g^{-1}}(X_i)\big),\psi_{g^{-1}}(Y_i)\Big)\Big]$$
    Because the group actions $\varphi_g$ and $\psi_g$ are $G$-isometries (measure-preserving transformations) and $S_f\in [0,b]$
    $$0 \leq s\Big(f_\theta\big(\phi_{g^{-1}}(X_i)\big),\psi_{g^{-1}}(Y_i)\Big) \leq b\quad \forall g \in G$$
    After averaging over the orbit, $b_G:=\sup_{(X_i,Y_i)}s(f_\theta^G(X_i),Y_i) \leq b$ and $\mathbb{B}_{\text{Hoeffding}}(\bar S_{n_c}) \geq \mathbb{B}_{\text{Hoeffding}}(\bar S^G_{n_c})$.
\end{proof}


\bibliographystyle{ieeetr}
\bibliography{bibliography}

\end{document}